\documentclass[journal]{IEEEtran}
\usepackage{amsmath,amsfonts}
\usepackage{algorithmic}
\usepackage{algorithm}
\usepackage{array}
\usepackage[caption=false,font=normalsize,labelfont=sf,textfont=sf]{subfig}
\usepackage{textcomp}
\usepackage{stfloats}
\usepackage{url}
\usepackage{verbatim}
\usepackage{graphicx}
\usepackage{cite}
\usepackage{multirow}
\usepackage{makecell}
\usepackage{subfig}
\usepackage{enumitem}
\usepackage{amsthm,amsmath,amssymb}
\usepackage{mathrsfs}
\usepackage{colortbl}

\theoremstyle{plain}
\newtheorem{theorem}{Theorem}[section]

\theoremstyle{definition}

\theoremstyle{remark}

\begin{document}

\title{MPQ-DMv2: Flexible Residual Mixed Precision Quantization for Low-Bit Diffusion Models with Temporal Distillation}
\author{Weilun Feng, Chuanguang Yang, Haotong Qin, Yuqi Li, Xiangqi Li, Zhulin An, Libo Huang, Boyu Diao, Fuzhen Zhuang,~\IEEEmembership{Member,~IEEE}, Michele Magno,~\IEEEmembership{Senior Member,~IEEE}, Yongjun Xu, Yingli Tian,~\IEEEmembership{Fellow,~IEEE} and Tingwen Huang,~\IEEEmembership{Fellow,~IEEE}
\thanks{Weilun Feng and Xiangqi Li are with the Institute of Computing Technology, Chinese Academy of Sciences, Beijing 100190, China, and also with University of Chinese Academy of Sciences, Beijing 100049, China (e-mail: fengweilun24s@ict.ac.cn; lixiangqi24@mails.ucas.ac.cn).}
\thanks{Chuanguang Yang, Yuqi Li, Zhulin An, Libo Huang, Boyu Diao, and Yongjun Xu are with the Institute of Computing Technology, Chinese Academy of Sciences, Beijing 100190, China (e-mail: yangchuanguang@ict.ac.cn; yuqili010602@gmail.com; anzhulin@ict.ac.cn; www.huanglibo@gmail.com; diaoboyu2012@ict.ac.cn; xyj@ict.ac.cn).}
\thanks{Haotong Qin and Michele Magno are with ETH Zurich, Gloriastrasse 35, 8092 Zürich, Switzerland (e-mail: haotong.qin@pbl.ee.ethz.ch; michele.magno@pbl.ee.ethz.ch).}
\thanks{Fuzhen Zhuang is with Institute of Artificial Intelligence, Beihang University, Beijing, China and Zhongguancun Laboratory, Beijing, China (e-mail: zhuangfuzhen@buaa.edu.cn).}
\thanks{Yingli Tian is with the Department of Electrical Engineering, The City College, and the Department of Computer Science, the Graduate Center, the City University of New York, New York, NY, 10031 (e-mail: ytian@ccny.cuny.edu).}
\thanks{Tingwen Huang is with the School of Computer Science and Control Engineering, Shenzhen University of Advanced Technology, Shenzhen 518107, China (e-mail: huangtingwen@suat-sz.edu.cn).}
\thanks{Corresponding authors: Zhulin An and Chuanguang Yang.}}

\markboth{Journal of \LaTeX\ Class Files,~Vol.~14, No.~8, August~2021}%
{Shell \MakeLowercase{\textit{et al.}}: A Sample Article Using IEEEtran.cls for IEEE Journals}


\maketitle

\begin{abstract}

Diffusion models have demonstrated remarkable performance on vision generation tasks. However, the high computational complexity hinders its wide application on edge devices. Quantization has emerged as a promising technique for inference acceleration and memory reduction. However, existing quantization methods do not generalize well under extremely low-bit (2-4 bit) quantization. Directly applying these methods will cause severe performance degradation. We identify that the existing quantization framework suffers from the outlier-unfriendly quantizer design, suboptimal initialization, and optimization strategy. We present MPQ-DMv2, an improved \textbf{M}ixed \textbf{P}recision \textbf{Q}uantization framework for extremely low-bit \textbf{D}iffusion \textbf{M}odels. For the quantization perspective, the imbalanced distribution caused by salient outliers is quantization-unfriendly for uniform quantizer. We propose \textit{Flexible Z-Order Residual Mixed Quantization} that utilizes an efficient binary residual branch for flexible quant steps to handle salient error. For the optimization framework, we theoretically analyzed the convergence and optimality of the LoRA module and propose \textit{Object-Oriented Low-Rank Initialization} to use prior quantization error for informative initialization. We then propose \textit{Memory-based Temporal Relation Distillation} to construct an online time-aware pixel queue for long-term denoising temporal information distillation, which ensures the overall temporal consistency between quantized and full-precision model. Comprehensive experiments on various generation tasks show that our MPQ-DMv2 surpasses current SOTA methods by a great margin on different architectures, especially under extremely low-bit widths.

\end{abstract}


\section{Introduction}

Diffusion Models (DMs)~\cite{ho2020ddpm, dhariwal2021diffusionbeatgan} have recently emerged as a powerful generative paradigm, demonstrating superior performance across a wide range of vision tasks, including image generation~\cite{rombach2022ldm, liu2023text}, video synthesis~\cite{mei2023vidm, hong2022cogvideo, liu2024sora}, image restoration~\cite{yang2024diffusion, wang2025osdface}, and audio generation~\cite{zhang2023audiosurvey}. The success largely stems from the iterative denoising mechanism of diffusion models, which denoise the generation target gradually from a Gaussian noise. This enables diffusion models to model complex data distributions with high fidelity. However, such iterative sampling also imposes significant computational burdens during inference, particularly when dealing with high-resolution data, as each forward pass must process a massive number of parameters repeatedly~\cite{croitoru2023diffusionsurvey}. This poses a major obstacle to deploying diffusion models in latency-sensitive or resource-constrained scenarios, such as mobile devices and edge computing environments~\cite{10.1093/comjnl/bxae110, DAI2024104811, huang2025foundation}.

\begin{figure}
    \centering
    \includegraphics[width=1.0\linewidth]{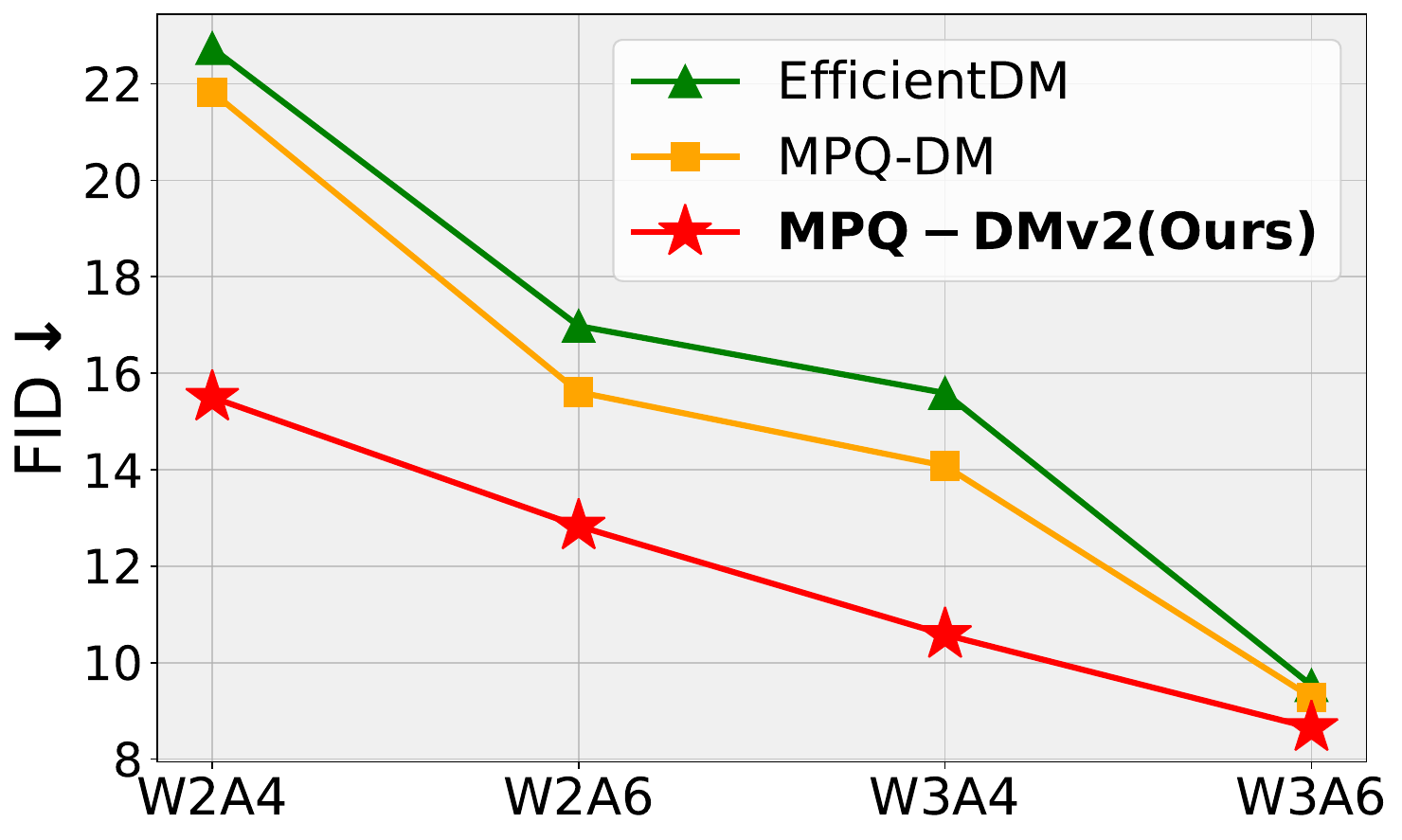}
    \caption{\textbf{The FID score for LDM-8 LSUN-Churches model under different quantization settings, lower FID indicates better performance. WxAy denotes x-bit weight and y-bit activation quantization, e.g., W2A4 denotes 2-bit weight and 4-bit activation quantization.} Our MPQ-DMv2 surpasses current quantization methods by a great margin.}
    \label{fig:teaser}
\end{figure}

To address the computational inefficiency of diffusion models, quantization has been introduced as an effective compression technique. By mapping full-precision floating-point representations to low-bit integers, quantization significantly reduces memory footprint and inference latency~\cite{gholami2022quantizationsurvey}. It has been extensively adopted in both convolutional network~\cite{pilipovic2018cnnquantsurvey, ding2024reg, chen2024scp}, Transformer~\cite{chitty2023transformerquantsurvey}, and Mamba~\cite{tianqi2025qmamba, guan2024qMamba}. Despite these successes, quantizing diffusion models remains highly challenging due to their unique denoising dynamics and the sequential accumulation of discretization errors over denoising time steps. Unlike deterministic models, the intermediate representations in diffusion models are more sensitive to quantization perturbation. Small quantization errors can be amplified across time steps, leading to quality degradation in the final output.

To mitigate this issue, quantization-aware training (QAT)~\cite{esser2019lsq, jacob2018quantizationandtrain, krishnamoorthi1806quantizingwhitepaper} has been explored in diffusion models~\cite{zheng2024binarydm, li2024qdm, zheng2024bidm}, allowing fine-tuning of weights and quantization parameters using large training data. While QAT often yields strong performance, especially under low-bit settings or even binarization, it requires expensive re-training resources comparable to the original model training. In contrast, post-training quantization (PTQ)~\cite{hubara2020ptqlayer, wang2020towards, wei2022qdrop, liu2023pdquant} provides a lightweight alternative by calibrating quantization parameters with a small subset of calibration data. This makes it possible to quantize the pre-trained models while utilizing a small amount of computing resources and maintaining high precision. However, existing PTQ-based methods typically struggle under extremely low bit-widths (e.g., 2-4 bits) due to the limited adaptation capacity.

\begin{figure*}[t]
    \centering
    \includegraphics[width=1.0\linewidth]{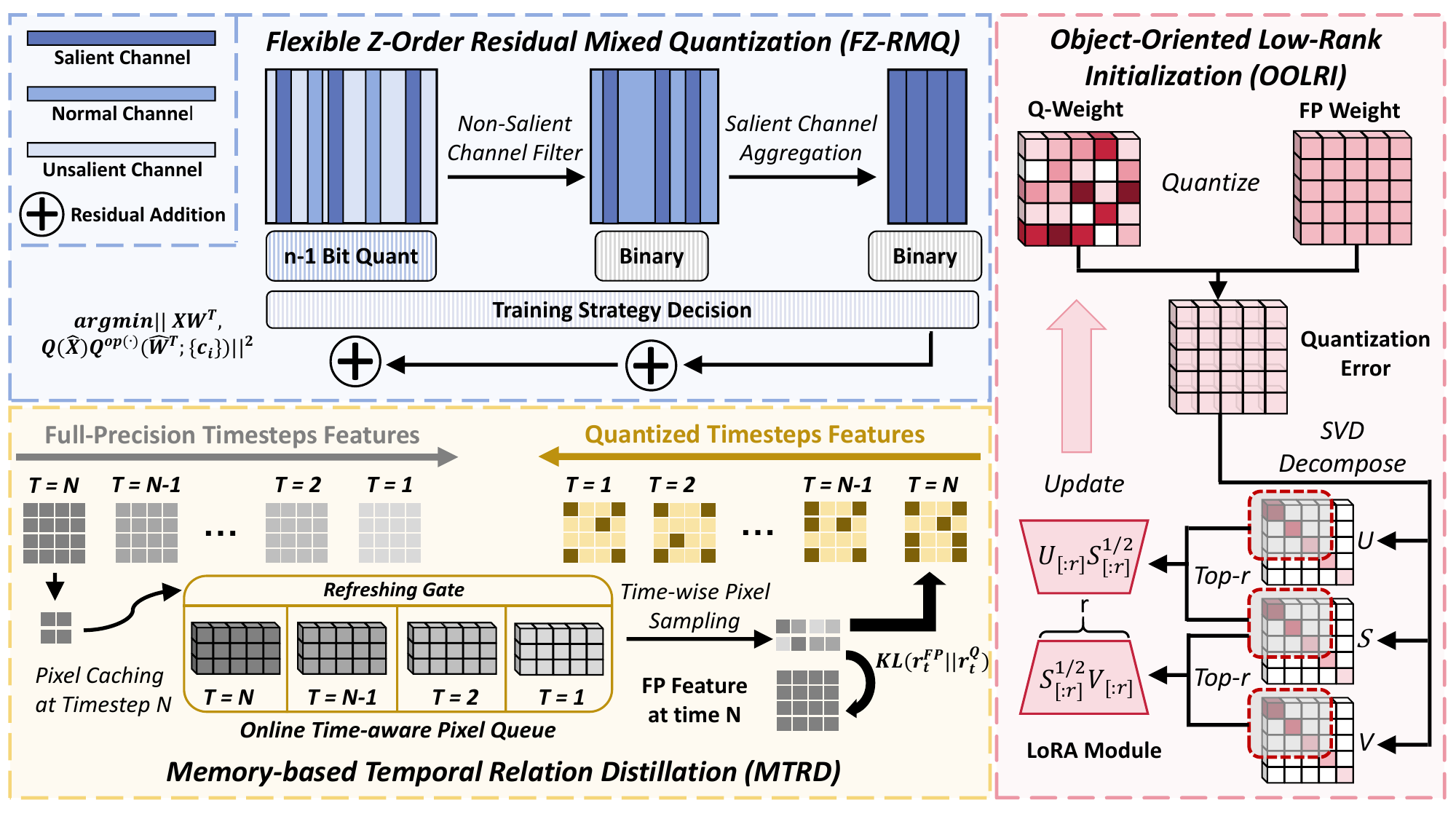}
    \caption{\textbf{Overview of proposed MPQ-DM2 framework.} The framework consists of \textit{Flexible Z-Order Residual Mixed Quantization} to use efficient binary branch for flexible quantizer design, \textit{Memory-based Temporal Relation Distillation} for denoising temporal consistency distillation, and \textit{Object-Oriented Low-Rank Initialization} to use prior quantization error for informative initialization.}
    \label{fig:overview}
\end{figure*}

To overcome the performance degradation, Quantization-Aware Low-Rank Adaptation (QA-LoRA) scheme~\cite{he2023efficientdm, feng2025mpqdm} is introduced into diffusion models. QA-LoRA uses LoRA~\cite{hu2021lora} technique to perform low-rank updates on the quantization weights, which does not affect the quantization process and makes the weight distribution more suitable for quantization operations, thereby improving the quantization performance. QA-LoRA can even achieve QAT-level performance under PTQ-like calibration budgets. While this approach significantly improves robustness, our empirical study reveals that it still suffers from noticeable performance degradation in extremely low-bit settings (e.g., 2-3 bits). We attribute this to three critical bottlenecks in existing quantization pipelines for diffusion models:

\noindent
\textbf{(1) Inflexible quantizer design.} Existing quantization frameworks~\cite{he2023efficientdm, feng2025mpqdm} adopt uniform quantization strategies for the model weights. However, we observe significant variance in channel-wise weight distributions, including sparse and salient outliers. Uniform quantizers fail to effectively allocate bit-widths for such heterogeneous statistics, leading to suboptimal encoding of critical weights.

\noindent
\textbf{(2) Isolated temporal supervision.} Existing methods~\cite{he2023efficientdm, feng2025mpqdm} use time-wise activation quantization strategy to utilize different quantization parameters for different denoising steps. While this strategy ensures parameter adaptation to variant activation distribution, it optimizes intermediate features independently at each timestep, ignoring the strong temporal correlation across denoising stages. This results in temporally inconsistent latent trajectories, which lack the perception of the overall denoising trajectory and result in significant deviation of the final denoising result.

\noindent
\textbf{(3) Cold-start low-rank adaptation.} QA-LoRA modules are often initialized with zero matrix statistics, neglecting the value prior information of the quantization error. This causes a cold-start optimization scenario where the module must learn corrections from scratch, thereby resulting in larger quantization error and slowing the optimization convergence.

To address the quantization limitations of extremely low-bit quantization, we propose a unified quantization framework, Flexible Residual \textbf{M}ixed \textbf{P}recision \textbf{Q}uantization for Low-Bit \textbf{D}iffusion \textbf{M}odels with Temporal Distillation (MPQ-DMv2). In Fig.~\ref{fig:teaser}, we compare our proposed MPQ-DMv2 with current SOTA methods. MPQ-DMv2 significantly surpasses existing methods by a large margin. The overview of the proposed framework is in Fig.~\ref{fig:overview}. The proposed MPQ-DMv2 framework mainly relies on three novel techniques: \textit{Flexible Z-Order Residual Mixed Quantization} (FZ-RMQ): utilizes a primary quantizer for the central weight distribution and a lightweight residual quantizer that uses binary codes to capture salient residuals. This residual pathway effectively preserves important outliers with minimal overhead. \textit{Memory-based Temporal Relation Distillation} (MTRD): introduces a memory-based online distillation mechanism that aligns temporal consistency by leveraging long-term structural dependencies. \textit{Object-Oriented Low-Rank Initialization} (OOLRI): initializes the low-rank module by approximating the prior quantization error using truncated singular value decomposition. Among these techniques, FZ-RMQ flexibly adapts to the presence of outliers under a unified mixed precision framework from the quantization architecture perspective. MTRD and OOLRI utilize quantization error as prior knowledge and consider overall denoising temporal correlation from the optimization perspective. Overall, we jointly improve the low-bit quantization performance of diffusion models from two low-bit perspectives: quantization architecture and optimization framework.

We summarize the main contributions of this paper as:

\begin{itemize}
    \item We present a residual quantization strategy that flexibly adjusts the quant scale to adapt to the distribution of salient outliers. We further propose a hierarchical extension that supports channel-wise mixed-precision framework with shared base quantizer and binary residual. The framework is guided by output discrepancy for bit allocation and optimization strategy selection, ensuring optimal quantization architecture search.
    
    \item We introduce a memory-based online distillation mechanism that aligns temporal consistency by leveraging long-term dependencies across denoising steps. Rather than independently supervising each timestep, we propagate temporal relational information from a memory-efficient bank of full-precision features, guiding the quantized model toward temporally coherent generations.
    
    \item We consider the initialization of QA-LoRA as a local optimization problem and theoretically proven the optimality and convergence of the solution. We propose to use quantization residual as prior knowledge and approximate the module using truncated singular value decomposition, generating informed low-rank corrections that serve as better initialization.

    \item Comprehensive experiments on Unet-based LDM and Stable Diffusion models and Transformer-based DiT-XL models show that our MPQ-DMv2 surpasses the existing SOTA methods by a wide margin on various architectures and bit-widths, which demonstrates the effectiveness and generalization of our method.
\end{itemize}

Note that this paper significantly extends our preliminary conference work~\cite{feng2025mpqdm}. The earlier method, MPQ-DM, introduced in the conference paper, served as an initial exploration into low-bit mixing precision quantization for diffusion models. In this paper, we present an enhanced and generalized version, MPQ-DMv2, which introduces substantial architectural and algorithmic improvements. We first analyze the bottlenecks of existing low-bit quantization frameworks in terms of quantization architecture and optimization methods. To address these, we propose FZ-RMQ, a flexible quantizer tailored to accommodate salient distributions under low-bit constraints. Additionally, we introduce MTRD and OOLRI, two novel optimization strategies designed to mitigate issues of temporal inconsistency and LoRA cold-start initialization, respectively.
Our enhanced scheme, MPQ-DMv2, is evaluated across a broader range of diffusion model architectures and bit-width settings compared to the original work. Furthermore, we provide in-depth analysis and discussion regarding the performance, robustness, and generalization of the proposed methods. MPQ-DMv2 provides new insight and sets a new benchmark for low-bit quantization in diffusion models.

\section{Related Work}

\subsection{Diffusion Model}
%
Diffusion models \cite{ho2020ddpm, song2020ddim} have achieved remarkable achievements in multiple fields, such as image generation~\cite{rombach2022ldm, peebles2023dit, yang2025multi}, super-resolution~\cite{wu2024osediff}, image restoration~\cite{wang2025osdface}, video generation~\cite{hong2022cogvideo, liu2024sora}, and audio generation~\cite{zhang2023audiosurvey}. The most unique mechanism of diffusion models is the iterative denoising process, which often requires dozens or even hundreds of repeated forward propagation. Therefore, many works to reduce the number of sampling steps have been proposed to accelerate the inference of diffusion models. Some efficient samplers~\cite{song2020ddim, liu2022plms, lu2022dpm, lu2022dpmm} directly reduce the number of denoising steps from the perspective of the sampling function. Some knowledge distillation works~\cite{salimans2022pd, song2023consistency, feng2024rdd} reduce step redundancy in the pre-trained model by retraining to distill a few-step diffusion model. However, these works focus on reducing the sampling steps of diffusion models, but still cannot solve the computational burden of diffusion models in a single step. This paper focuses on compressing the storage of diffusion models and reducing the burden of single inference.

\subsection{Model Quantization}
Model quantization~\cite{gholami2022quantizationsurvey, jacob2018quantizationandtrain, krishnamoorthi1806quantizingwhitepaper} is an effective compression and acceleration method that quantizes full-precision data into low-bit data. Among them, it can be divided into Quantization-Aware Training (QAT)~\cite{esser2019lsq, feng2025qvdit, qin2023diverse} and Post-Training Quantization (PTQ)~\cite{krishnamoorthi1806quantizingwhitepaper, ding2024reg, gong2025pushing} based on whether the model is fully re-trained. QAT requires a large amount of training data to fine-tune the model weights and quantization parameters, often achieving good performance under ultra-low bit width or even binarization. However, its enormous computing resources and training time are the main bottlenecks of QAT. PTQ has received more attention as an efficient method for fine-tuning only the quantization parameters of the model. PTQ only requires a small amount of calibration data to fine-tune the quantization parameters of the model and preserve its good performance after quantization. However, severe performance collapse often occurs at extremely low bits. The existing quantization methods mainly design different quantization methods for different model architectures, such as CNN~\cite{pilipovic2018cnnquantsurvey, dong2019hawq}, Transformer~\cite{chitty2023transformerquantsurvey, xiao2023smoothquant}, Mamba~\cite{tianqi2025qmamba, guan2024qMamba}, and Diffusion Models~\cite{li2023qdiffusion}.

\subsection{Diffusion Quantization}
Diffusion models have brought new challenges to quantization due to the unique iterative denoising mechanism. Therefore, quantization methods targeting diffusion models have been proposed from different perspectives. For PTQ methods, PTQ4DM \cite{shang2023ptq4dm} and Q-Diffusion \cite{li2023qdiffusion} have made initial exploration from the perspective of quantization objective and frameworks. The following studies PTQ-D \cite{he2024ptqd}, TFMQ-DM \cite{huang2024tfmq}, and APQ-DM \cite{wang2024apqdm}, have made improvements in the direction of cumulation of quantization error, temporal feature maintenance, and calibration data construction. Moreover, PTQ4DiT~\cite{wu2024ptq4dit} also proposes a quantization method to adapt to the unique data distribution for the Diffusion Transformer's special framework. However, the performance of PTQ-based methods still suffers from severe degradation at extremely low bit-width (e.g., under 4-bit). Therefore, QAT methods for diffusion models have been proposed. Q-dm~\cite{li2024qdm} maintains the quantization performance of the diffusion model at 2-3 bits. BinaryDM~\cite{zheng2024binarydm} and BiDM~\cite{zheng2024bidm} utilize fine-grained quantizer design and distillation methods to preserve the performance of diffusion models in extreme binary quantization. However, these QAT methods require a lot of extra training time compared with PTQ methods, resulting larger training burden.

To combine the advantages of QAT and reduce the required training time, QuEST~\cite{wang2024quest} only fine-tunes the weights of sensitive layers to reduce fine-tuning burden and improve performance. EfficientDM \cite{he2023efficientdm} uses Quantization-Aware Low-Rank Adaptation (QA-LoRA) to fine-tune the quantized diffusion model with PTQ level low computing resources while maintaining high precision. On this basis, MPQ-DM~\cite{feng2025mpqdm} uses inter-channel mixed precision quantization to further improve the quantization performance at ultra-low bit widths. However, we identify that the existing quantization frameworks have inherent shortcomings. This includes the unfriendliness of uniform quantizers towards outliers, insufficient temporal information maintenance ability of time-wise quantization strategies, and the zero initialization of LoRA module not utilizing prior knowledge of quantization errors. Therefore, we propose MPQ-DMv2, which retains the existing mixed precision quantization framework and makes improvements from the above three perspectives.

\section{Preliminary}

\subsection{Model Quantization}

Model quantization~\cite{gholami2022quantizationsurvey, jacob2018quantizationandtrain} maps floating-point model weights and activations to low-bit integer representations to reduce memory consumption and accelerate inference. For a floating-point vector $\mathbf{x}_f$, the quantization process is formally defined as:
\begin{equation}
\begin{gathered}
    \mathbf{x}^Q = Q(\mathbf{x}, s, z) = s \cdot \left[ \text{clip}\left(\left\lfloor \frac{\mathbf{x}_f}{s} \right\rceil + z, 0, 2^N{-}1\right) - z \right], \\
    s = \frac{u - l}{2^N - 1}, \quad z = - \left\lfloor \frac{l}{s} \right\rceil,
\end{gathered}
\end{equation}
where $\mathbf{x}^Q$ denotes the quantized vector, $\left\lfloor \cdot \right\rceil$ represents rounding to the nearest integer, and $\text{clip}(\cdot)$ clamps the value into the range $[0, 2^N{-}1]$. Here, $s$ is the scale factor, and $z$ is the quantization zero-point. The values $l$ and $u$ denote the lower and upper bounds of the quantization range, respectively, which are determined by the distribution of $\mathbf{x}$ and the target bit-width $N$.

\subsection{Channel-wise Pre-scaling for Quantization}

In diffusion models, weight distributions across channels can exhibit significant heterogeneity~\cite{feng2025mpqdm, li2023qdiffusion, wu2024ptq4dit}, with some channels containing extreme outliers that severely degrade the efficacy of uniform quantization. To address this issue, existing methods~\cite{feng2025mpqdm, wu2024ptq4dit} introduce a channel-wise scaling mechanism that preconditions each input channel to suppress the influence of outliers.

Given a weight matrix $\mathbf{W} \in \mathbb{R}^{C_{\text{out}} \times C_{\text{in}}}$, where $C_{\text{out}}$ and $C_{\text{in}}$ denote the output and input channels, respectively, a channel-wise scaling factor $\delta_i$ is computed for each input channel $i$ based on calibration data. Specifically, for a calibration activation matrix $\mathbf{X} \in \mathbb{R}^{C_{\text{in}} \times d}$, the per-channel scaling factor is defined as:
\begin{equation}
\delta_i = \sqrt{\frac{\max(|\mathbf{W}_i|)}{\max(|\mathbf{X}_i|)}},
\end{equation}
where $\mathbf{W}_i$ and $\mathbf{X}_i$ denote the $i$-th column of $\mathbf{W}$ and $\mathbf{X}$, respectively. This factor is used to normalize the input and weight distributions, thereby reducing the prominence of outliers and improving quantization fidelity.

To preserve computational equivalence, both the input and weights are pre-scaled in a way that maintains the original matrix product:
\begin{equation}
\begin{gathered}
    \mathbf{X}\mathbf{W}^{\top} = \hat{\mathbf{X}}\hat{\mathbf{W}}^{\top}, \\
    \text{where} \quad \hat{\mathbf{X}} = \mathbf{X} \cdot \text{diag}(\delta), \quad \hat{\mathbf{W}} = \text{diag}^{-1}(\delta) \cdot \mathbf{W}.
\end{gathered}
\end{equation}

\subsection{Channel-wise Mixed-Precision Quantization}

Conventional quantization methods~\cite{he2023efficientdm, li2023qdiffusion} typically employ a uniform bit-width across all channels within a given layer. However, in diffusion models, weight distributions can vary dramatically across channels~\cite{feng2025mpqdm, wu2024ptq4dit}. Some exhibit heavy-tailed behavior with substantial outliers, which causes severe performance degradation under low bit quantization.

To mitigate this limitation, MPQ-DM~\cite{feng2025mpqdm} employs a channel-wise mixed-precision quantization (MPQ) scheme. In this strategy, the quantization bit-width is adaptively assigned to each channel based on its statistical properties. Specifically, the kurtosis of each channel’s weight distribution is used to estimate its “tailedness” which naturally quantifies the sensitivity to quantization. For a weight matrix $\mathbf{W} \in \mathbb{R}^{C_{\text{out}} \times C_{\text{in}}}$, the kurtosis for input channel $i$ is calculated as:
\begin{equation}
    \kappa_i = \frac{\mathbb{E}[(\mathbf{W}_i - \mu)^4]}{\sigma^4}, \quad \mu = \mathbb{E}[\mathbf{W}_i],\quad \sigma^2 = \text{Var}[\mathbf{W}_i],
\end{equation}
where $\mathbf{W}_i$ is the weight vector of channel $i$. Higher values of $\kappa_i$ indicate the presence of more significant outliers, which justify the use of higher bit-widths for those channels.

Given a target average bit-width $N$, MPQ-DM formulates an optimization problem to determine a bit-width vector $[c_1, \ldots, c_n]$ while minimizing the quantization error:
\begin{equation}
\begin{aligned}
    \min_{\{c_i\}} \quad & \left\| \mathbf{X}\mathbf{W}^{\top} - Q(\hat{\mathbf{X}})Q(\hat{\mathbf{W}}^{\top} | \{c_i\}) \right\|_2^2, \\
    \text{s.t.} \quad & c_i \in \{N{-}1, N, N{+}1\},\quad \sum_{i=1}^{n} c_i = nN,
\end{aligned}
\end{equation}
where $c_i$ denotes the quantization bit-width for channel $i$.

\section{Method}

\subsection{Flexible Z-Order Residual Mixed Quantization}

\subsubsection{Architecture Formulation}
Existing mixed-precision quantization strategies, such as MPQ-DM~\cite{feng2025mpqdm}, attempt to preserve model performance by allocating varying bit-widths across weight channels according to statistical difficulty metrics (e.g., kurtosis). Specifically, channels with higher kurtosis which indicate a higher presence of outliers are assigned more bits to mitigate the quantization error. However, existing frameworks~\cite{ma2023ompq, dong2019hawq, feng2025mpqdm} only focus on the mixed precision bit allocation strategy, while neglecting the equally important quantizer design. Existing frameworks rely solely on uniform quantizers, which are inherently limited in representing distributions with extreme and sparse outliers. Though few, these outliers can dominate the quantization range, leading to suboptimal quantization performance. We visualize the outlier-existing weight distribution on LDM-4 model in Fig.~\ref{fig:residual_vs_uniform}. It can be seen in Fig.~\ref{fig:uniform_quant} that for the uniform quantizer, significant bits are wasted on rarely occurring values, thereby degrading the precision for the core of the distribution.

To overcome this limitation, we propose a novel quantization framework called \textbf{Flexible Z-Order Residual Mixed Quantization (FZ-RMQ)}. The key idea is to decouple the quantization of dense distribution cores and sparse outliers using a dual-quantizer design: a \textit{main quantizer} for the concentrated majority and a lightweight \textit{residual quantizer} for the high-magnitude quantization residuals. This enables more expressive, fine-grained representation of weights with minimal computational overhead.

\subsubsection{Residual Quantization with Adaptive Step Modulation}

For residual quantization, we use a residual quantizer to compensate for the difference between the quantization target and the main quantizer. By adding the residual quantizer and the main quantizer, we can learn more flexible quantization steps while keeping the total number of bits constant. Formally, let $\mathbf{x} \in \mathbb{R}^d$ denote a floating-point weight vector, and let $Q_n(\cdot)$ represent a standard $n$-bit uniform quantizer with learnable step size $\Delta$. Instead of directly applying $Q_n$ to $\mathbf{x}$, we decompose the quantization as follows:
\begin{equation}
\begin{gathered}
\hat{Q}_n(\mathbf{x}) = Q_{n_1}(\mathbf{x}) + Q^{\mathrm{res}}_{n_2}(\mathbf{x} - Q_{n_1}(\mathbf{x})), \\
\text{s.t.} \quad n=n_1+n_2,
\end{gathered}
\label{eq:first_res_quant_scheme}
\end{equation}
where $Q_{n_1}$ denotes the main quantizer, and $Q^{\mathrm{res}}_{n_2}$ is a residual quantizer applied to the quantization error. In theory, there are multiple options for the values of $n_1$ and $n_2$. In our implementation, in order to make the residual part more efficient and preserve the expressive power of the main quantizer, we typically set $n_1 = n - 1$ and $n_2 = 1$. This design simplifies the residual quantizer into a binary quantizer, which can achieve inference acceleration~\cite{zhang2019dabnn} using efficient bitwise operations such as XNOR and bitcount to replace matrix multiplication. In this way, Eq.~\eqref{eq:first_res_quant_scheme} can be rewritten as:
\begin{equation}
\begin{gathered}
Q^{\mathrm{res}}_1(r) = \Delta_{\mathrm{res}} \cdot \mathrm{sign}(r), \quad r \in \mathbb{R}, \\
\hat{Q}_n(\mathbf{x}) = Q_{n-1}(\mathbf{x}) + Q^{\mathrm{res}}_{1}(\mathbf{x} - Q_{n_1}(\mathbf{x})),
\end{gathered}
\end{equation}
where $\Delta_{\mathrm{res}}$ is a learnable residual step size.

This construction modulates the quantization granularity in a data-adaptive manner. The main quantizer preserves fine-grained precision for the central part of the distribution, while the residual term allows for efficient handling of sparse outliers via a binary operation. The combined quantization step size becomes piecewise non-uniform, effectively expanding the quantization dynamic range while preserving core fidelity. We visualized the effect of residual quantization in Fig.~\ref{fig:res_quant}.

\begin{figure}[t]
    \centering
    \subfloat[][Uniform Quantizer]{
        \includegraphics[width=0.47\linewidth]{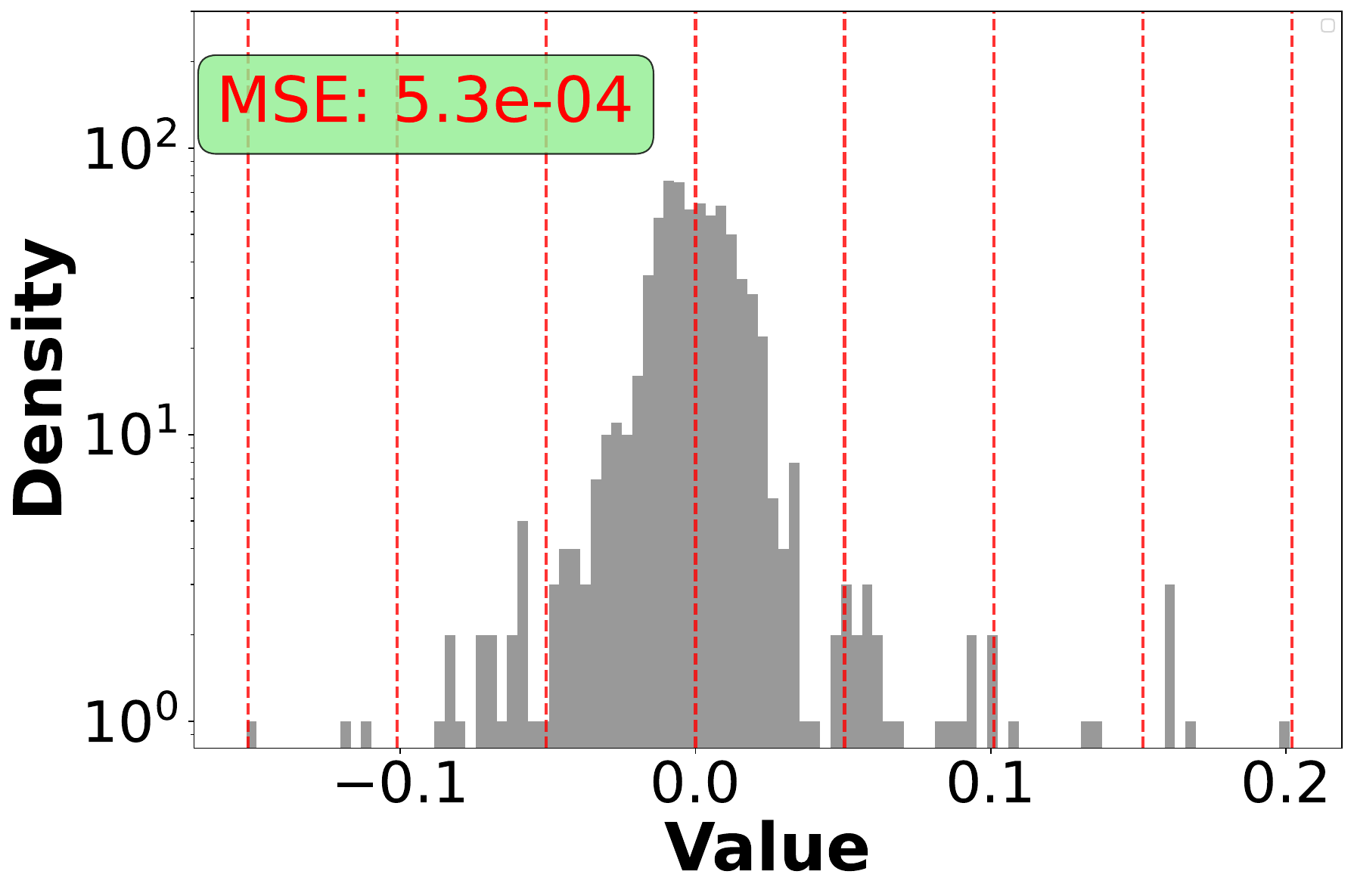}
        \label{fig:uniform_quant}
    }
    \subfloat[][Residual Quantizer]{
        \includegraphics[width=0.47\linewidth]{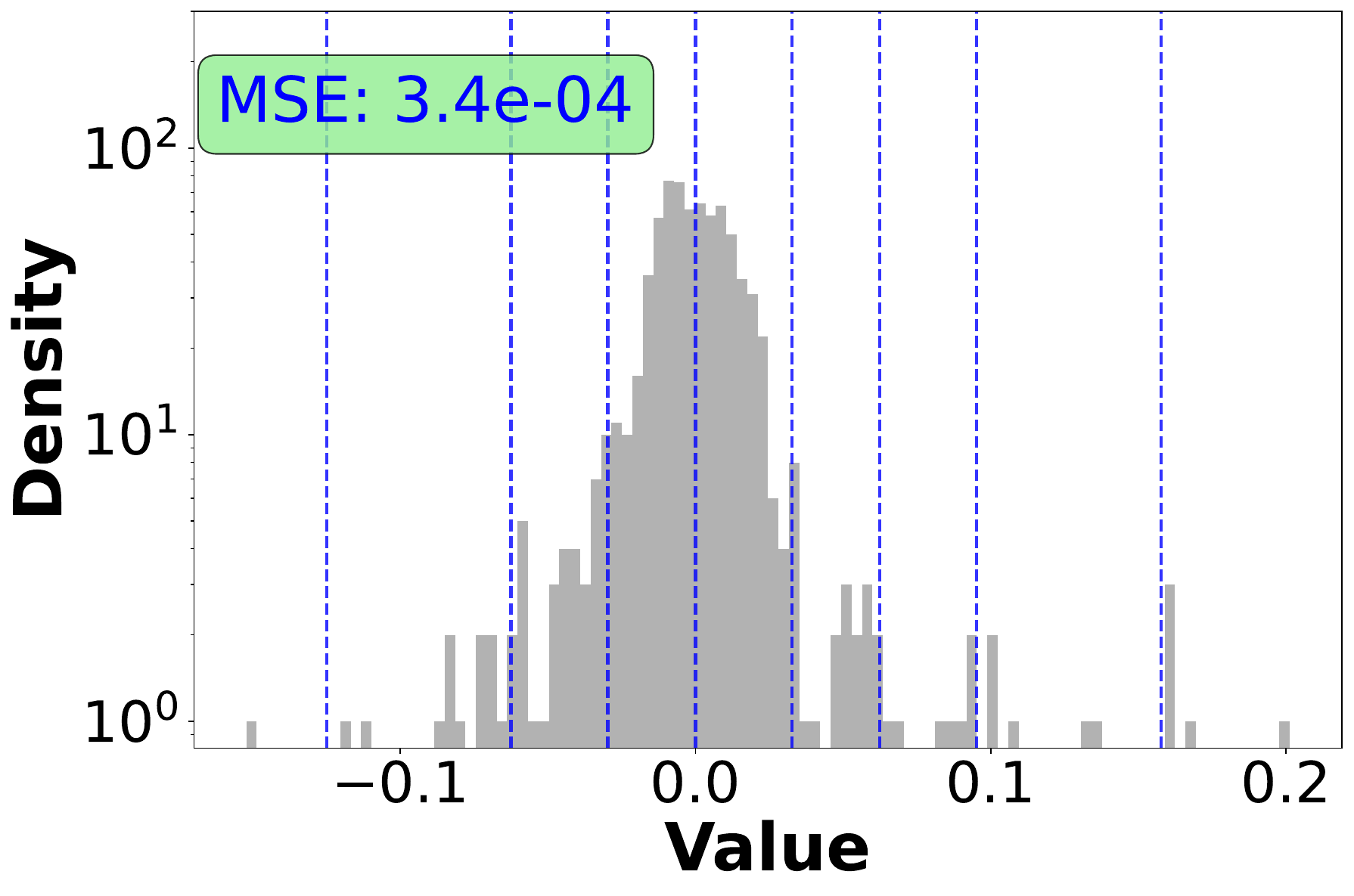}
        \label{fig:res_quant}
    }
    \caption{Comparison of 3-bit quantization step size distribution on LDM-4 model input\_blocks[1].emb\_layer channel 44. We label the corresponding quantization error on the top left part. The residual quantizer can flexibly adjust the quant step for non-uniform distribution, resulting in less quantization error.}
    \label{fig:residual_vs_uniform}
\end{figure}

\subsubsection{Unified Hierarchical Residual Mixed Precision}

To maintain compatibility with the mixed-precision design philosophy of existing frameworks~\cite{feng2025mpqdm} while extending its flexibility, we adopt a hierarchical residual quantization strategy for a channel-wise weight matrix $\hat{\mathbf{W}} \in \mathbb{R}^{c_{out}\times c_{in}}$:
\begin{equation}
\begin{gathered}
    \begin{cases}
        \hat{Q}_{n-1}(\mathbf{\hat{\mathbf{W}}}_i)=Q_{n-1}(\mathbf{\hat{\mathbf{W}}}_i),\\
        \hat{Q}_{n}(\mathbf{\hat{\mathbf{W}}}_i)=\hat{Q}_{n-1}(\mathbf{\hat{\mathbf{W}}}_i) + Q_1^{\mathrm{res}}(\hat{\mathbf{W}}_i - Q_{n-1}(\hat{\mathbf{W}}_i)),\\
        \hat{Q}_{n+1}(\mathbf{\hat{\mathbf{W}}}_i)=\hat{Q}_{n}(\mathbf{\hat{\mathbf{W}}}_i) + Q_1^{\mathrm{res},2}(\hat{\mathbf{W}}_i - \hat{Q}_{n}(\mathbf{\hat{\mathbf{W}}}_i)),
    \end{cases}\\
    \hat{Q}_{\mathrm{res}}(\hat{\mathbf{W}}) = \hat{Q}_{c_i}(\hat{\mathbf{W};i}),
\end{gathered}
\end{equation}
where $\hat{Q}_{c_i}(\hat{\mathbf{W};i})$ denotes channel-wise quantization for $\hat{\mathbf{W}}$, $c_i$ denotes the assigned bit-width for channel $i$. This hierarchical construction allows all channels to share the same base quantizer $Q_{n{-}1}$, thereby preserving quantization consistency across the entire tensor. The higher bit-widths ($n$ or $n{+}1$) are simply extended via lightweight binarized residuals, which can be implemented by binary channel mask and tensor addition.

\subsubsection{Flexible Optimization Strategies: Joint vs. Separate Training}
\label{subsec:quant_optimization}

The residual quantizer effectively solves the problem of distribution quantization with extreme outliers, but the reduction of bits in the main quantizer may affect the expression ability of partially uniform distributions. In order to preserve the powerful expressive power of the high-bit main quantizer, we propose two training paradigms to optimize the dual-quantizer configuration:

\begin{itemize}
    \item \textbf{Separate Optimization:} The main and residual quantizers are independently parameterized with their own learnable step sizes. This allows greater expressiveness and adaptivity, particularly useful for channels with extreme salient outliers.
    
    \item \textbf{Joint Optimization:} The main and residual quantizers are co-optimized with shared parameters (e.g., unified step size). This approach offers stronger expressiveness for uniform distribution and enables efficient inference by storing binary quantizer parameters as bit-wise shifts of a common base quantizer.
\end{itemize}

To determine which strategy is more beneficial for variant channel weight distribution, we define an activation-aware metric based on output distortion:
\begin{equation}
\begin{gathered}
    \underset{op(\cdot)}{\arg\min} \left\| \mathbf{X} \mathbf{W}^\top - Q(\hat{\mathbf{X}}) \cdot Q^{op(\cdot)}(\hat{\mathbf{W}}^\top) \right\|^2, \\
    \text{s.t.} \quad Q^{op(\text{joint})} = Q(\hat{\mathbf{W}}^{\top}), \quad
    Q^{op(\text{separate})} = \hat{Q}_{\mathrm{res}}(\hat{\mathbf{W}}^{\top}).
\end{gathered}
\end{equation}
This formulation allows data-driven selection between the two optimization paradigms based on minimum quantization-induced output error. For different channel distributions, we flexibly choose our optimization method without affecting the proposed residual quantization framework.

\subsubsection{Search Objective and Group-wise Channel Assignment}

Given the dual-bit design and optimization modes, we formulate the overall channel-wise bit allocation and quantizer selection problem as:
\begin{equation}
\begin{aligned}
    \underset{\{c_i, op_i\}}{\arg\min} \quad & \left\| \mathbf{X} \mathbf{W}^\top - Q(\hat{\mathbf{X}}) \cdot Q^{op(\cdot)}(\hat{\mathbf{W}}^\top; \{c_i\}) \right\|^2, \\
    \text{s.t.} \quad & c_i \in \{n{-}1, n, n{+}1\},\quad \sum_{i=1}^{C} c_i = C \cdot n,
\label{eq:quant_search_objective}
\end{aligned}
\end{equation}
where $C$ is the total number of channels and $Q^{op(\cdot)}$ is selected from the two optimization schemes in Sec.~\ref{subsec:quant_optimization}. To reduce the search space and facilitate efficient deployment, we partition channels into $g$ groups (empirically $g = C/10$) and perform group-wise bit-width assignment under fixed average budget constraints. This strategy strikes a practical balance between quantization granularity and computational cost.

Through Eq.~\eqref{eq:quant_search_objective}, we flexibly use three schemes: channel-wise mixed precision bit allocation, residual quantization design, and residual optimization strategy selection to minimize activation-aware quantization loss and achieve better quantization performance.

\subsection{Memory-based Temporal Relation Distillation}

\subsubsection{Architecture Formulation}

Existing diffusion model quantization methods like EfficientDM~\cite{he2023efficientdm} and MPQ-DM~\cite{feng2025mpqdm} typically adopt a \textit{time-wise activation quantization} strategy. Specifically, during both training and inference, they use a separate set of activation quantization parameters $(s_t, z_t)$ for each denoising timestep $t$. This approach aims to account for the statistical distribution shift of activation values across different stages of the diffusion trajectory~\cite{li2023qdiffusion, huang2024tfmq}. Formally, for the activation tensor $\hat{\mathbf{X}}_t$ at timestep $t$, the quantized version is defined as:
\begin{equation}
    \hat{\mathbf{X}}^Q_t = Q(\hat{\mathbf{X}}_t; s_t, z_t),\quad t \in [1,\dots,T],
\end{equation}
where $s_t$ and $z_t$ denote the timestep-specific scale and zero-point, respectively, and $T$ is the total number of denoising steps. The typical training loss aligns the quantized output with its full-precision (FP) counterpart:
\begin{equation}
    \mathcal{L}_{\text{align}} = \left\| \mathbf{X}_t \mathbf{W}^\top - Q(\hat{\mathbf{X}}_t; s_t, z_t) Q^{op(\cdot)}(\hat{\mathbf{W}}^\top) \right\|^2.
\end{equation}

While this per-timestep quantization improves accuracy by adapting to local distributional characteristics, it inherently treats each timestep as an independent optimization target. This independence ignores the temporal correlations embedded in the denoising process, where consecutive timesteps form a continuous trajectory rather than isolated steps. Consequently, the quantization optimization lacks a global temporal perspective, potentially weakening its ability to preserve the overall denoising semantics.

To address this limitation, we propose \textbf{Memory-based Temporal Relation Distillation (MTRD)}. MTRD introduces a global memory mechanism that captures the latent inter-timestep relationships by distilling relational knowledge across the diffusion timeline. It enables each timestep to be optimized not only based on its individual FP supervision but also through relational alignment with other steps. This holistic view effectively encodes the temporal structure of the diffusion process, leading to more temporally coherent quantized representations.

\begin{figure}[t]
    \centering
    \subfloat[][w/o MTRD]{
        \includegraphics[width=0.30\linewidth]{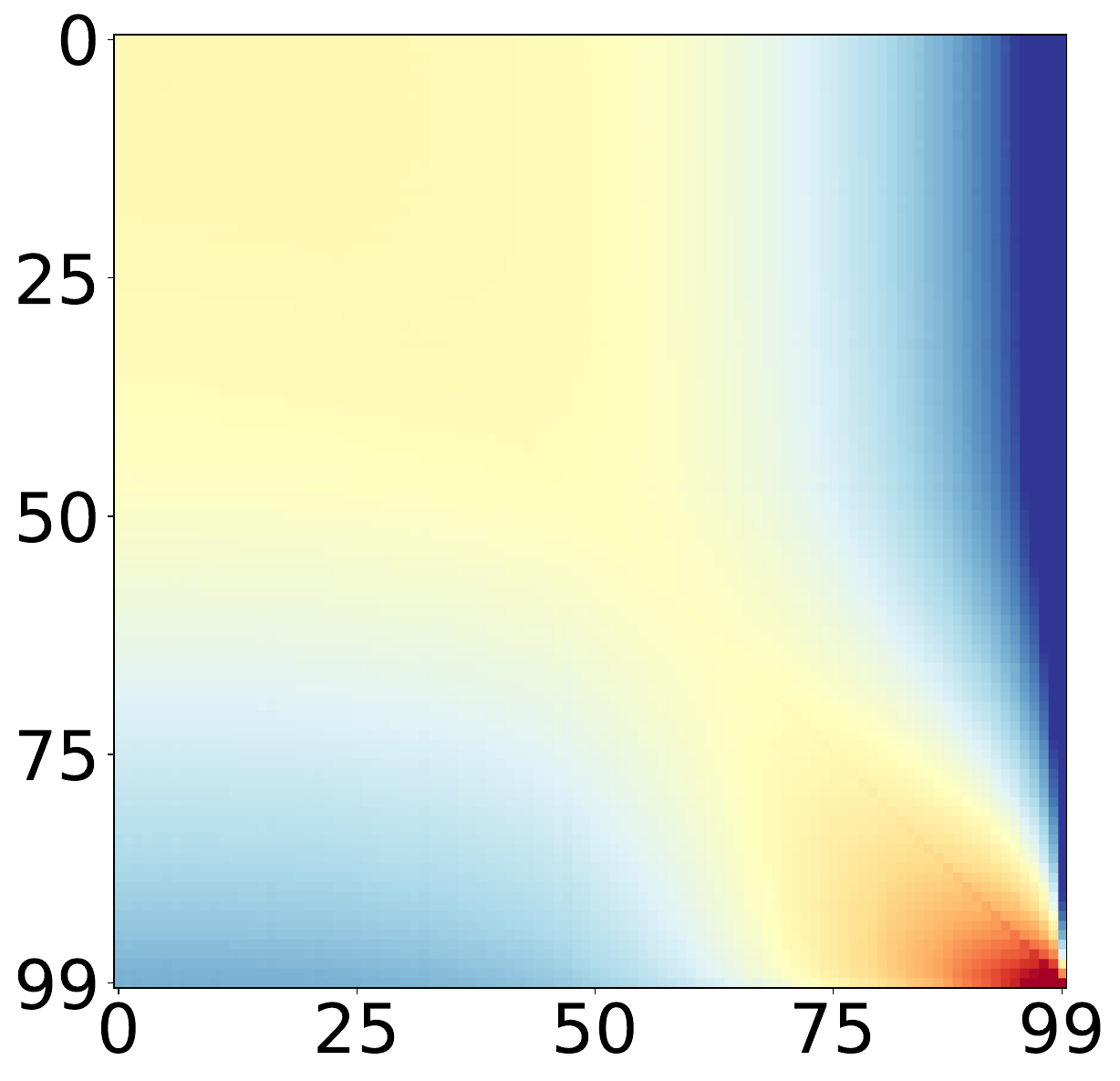}
    }
    \subfloat[][w MTRD]{
        \includegraphics[width=0.30\linewidth]{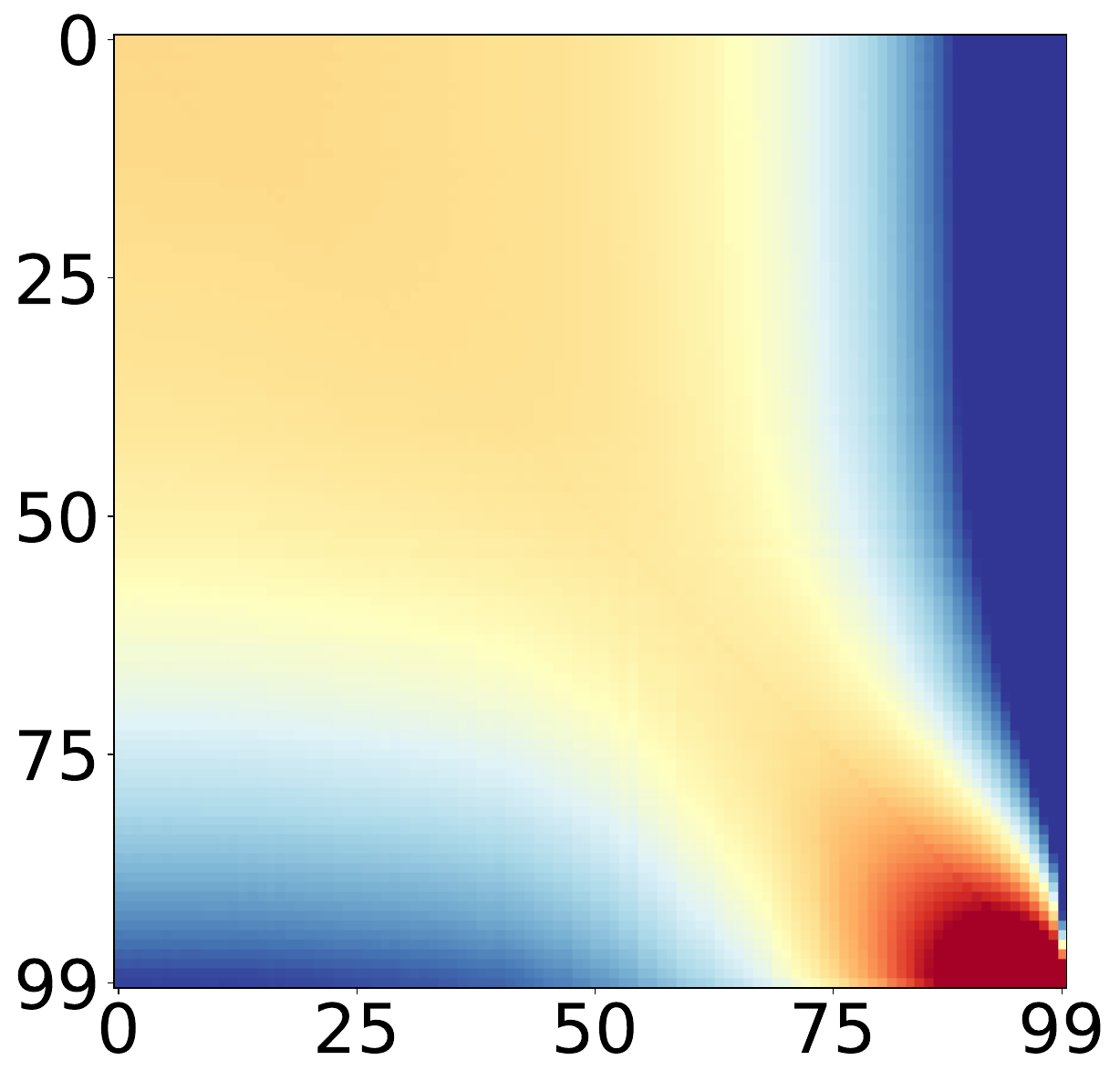}
    }
     \subfloat[][Full-Precision]{
        \includegraphics[width=0.30\linewidth]{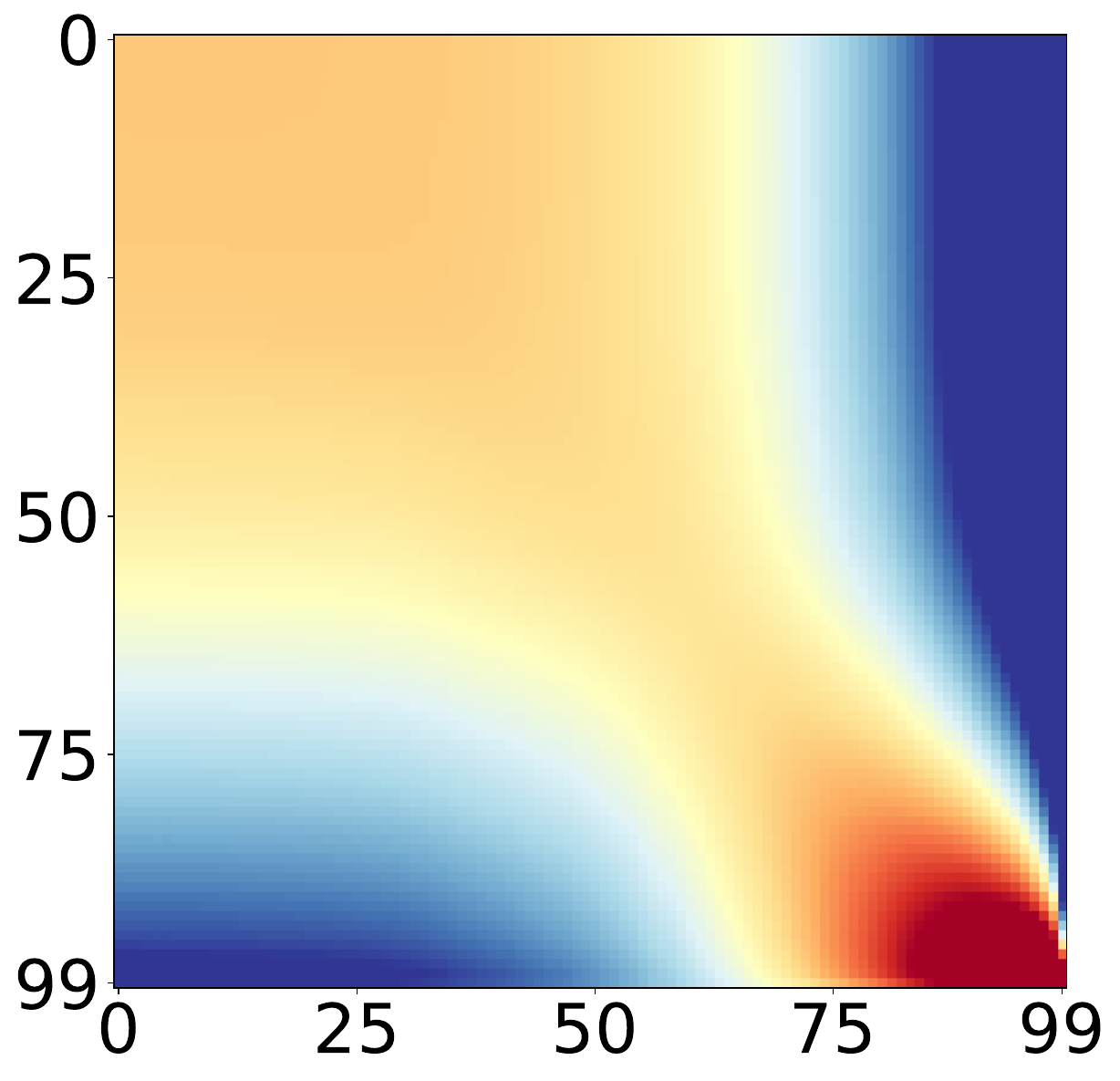}
    }
    \caption{Cosine similarity heatmaps across timesteps in the denoising process on LDM-8 under W3A6 quantization. ``w/o MTRD'' denotes without using MTRD technique. MTRD enhances temporal global relational consistency.}
    \label{fig:time_relation_mtrd}
\end{figure}

\subsubsection{Temporal Memory Queue}

To support global temporal modeling, we introduce a temporal memory queue $\mathcal{Q}$, which caches representative FP features from different timesteps as distillation anchors. This queue enables each timestep to be aware of all other timesteps relations during optimization. Specifically, we maintain $T$ independent first-in-first-out (FIFO) queues:
\begin{equation}
    \mathcal{Q} = \left\{ \mathcal{Q}_1, \mathcal{Q}_2, \dots, \mathcal{Q}_T \right\},\quad \mathcal{Q}_t \in \mathbb{R}^{L \times d},
\label{eq:queue_size}
\end{equation}
where $\mathcal{Q}_t$ stores up to $L$ feature vectors of dimension $d$ sampled from timestep $t$'s FP outputs.

At each training iteration involving timestep $t$, we randomly select a subset of $n$ feature vectors $\{\mathbf{x}_t^{(i)}\} \in \mathbb{R}^{n \times d}, n \ll s$ from the FP activation $\mathbf{X}_t \in \mathbb{R}^{s \times d}$ and push them into $\mathcal{Q}_t$. This ensures the memory remains lightweight while still maintaining a representative set of temporal features.

\subsubsection{Reference Relation Matrix Construction}

To embed temporal relational awareness into the distillation process, we construct a global reference matrix $\mathbf{F}_{\text{ref}}$ by sampling $k$ features from each queue $\mathcal{Q}_t$:
\begin{equation}
\begin{gathered}
    \mathbf{F}_{\text{ref}} = \left[ \mathbf{f}_1, \mathbf{f}_2, \dots, \mathbf{f}_R \right] \in \mathbb{R}^{R \times d}, \quad R = T \cdot k, \\
    \text{where}~\mathbf{f}_t = \{ \mathbf{x}_t^{(j)} \} \in \mathbb{R}^{k \times d}~\text{are sampled from}~\mathcal{Q}_t.
\label{eq:queue_sample_size}
\end{gathered}
\end{equation}

For a given feature $\hat{\mathbf{x}}_t$ from the quantized activation and its FP counterpart $\mathbf{x}_t^{\text{FP}}$, we compute their similarity with the reference matrix to obtain two relational distributions:
\begin{align}
    \mathbf{r}_t^{\text{FP}} &= \text{Softmax} \left( \frac{ \mathbf{F}_{\text{ref}} \cdot \mathbf{x}_t^{\text{FP}} }{ \tau } \right), \\
    \mathbf{r}_t^{\text{Q}} &= \text{Softmax} \left( \frac{ \mathbf{F}_{\text{ref}} \cdot \hat{\mathbf{x}}_t }{ \tau } \right),
\end{align}
where $\tau$ is a temperature parameter that controls the softness of the distribution.

These relational distributions encode how the feature relates to global temporal context~\cite{yang2023online, feng2024rdd, jing2020self}. The quantized feature is encouraged to mimic the relational structure of the FP one, rather than only matching it in feature space.

\subsubsection{Time-aware Distillation Loss}

To enforce the relational alignment between quantized and FP representations, we employ a Kullback-Leibler (KL) divergence loss:
\begin{equation}
    \mathcal{L}_{\text{MTRD}}^{(t)} = \text{KL}\left( \mathbf{r}_t^{\text{FP}} \, \| \, \mathbf{r}_t^{\text{Q}} \right).
\label{eq:KL_loss}
\end{equation}
This loss quantifies the discrepancy in how each feature relates to global temporal reference features.

The total distillation loss is accumulated over a mini-batch of timesteps:
\begin{equation}
    \mathcal{L}_{\text{MTRD}} = \frac{1}{|B|} \sum_{t \in B} \mathcal{L}_{\text{MTRD}}^{(t)},
\end{equation}
where $B$ denotes the set of sampled timesteps in the current training batch.

\subsubsection{Overall Optimization Loss}

We integrate our relation-aware distillation into the overall optimization objective:
\begin{equation}
    \mathcal{L}_{\text{total}} = \mathcal{L}_{\text{align}} + \alpha \mathcal{L}_{\text{MTRD}},
\label{eq:total_loss}
\end{equation}
where $\alpha$ is a balancing hyperparameter that controls the trade-off between standard alignment loss and temporal relational regularization.

In summary, MTRD complements the timestep-wise quantization by enabling each step to be optimized in light of the entire denoising trajectory. This approach effectively mitigates the local-isolation problem of conventional methods and enhances the temporal consistency of quantized diffusion models. We visualize the time-wise correlation heatmaps of all features across denoising timesteps in LDM-8 models in Fig.~\ref{fig:time_relation_mtrd}. With MTRD, the quantized correlation heatmaps are more similar to the full precision model, demonstrating their ability to preserve temporal information.

\subsection{Object-Oriented Low-Rank Initialization (OOLRI)}

\subsubsection{Architecture Formulation}
Existing frameworks~\cite{he2023efficientdm, feng2025mpqdm} utilize Quantization-aware low-rank adaptation (QA-LoRA) as an effective technique to recover performance loss introduced by low-bit quantization. Specifically, QA-LoRA injects trainable low-rank perturbations $\Delta \mathbf{W} = L_1 L_2$ into the quantized weights:
\begin{equation}
    Q(\mathbf{W}^*) = Q(\mathbf{W+\Delta \mathbf{W}}) = Q(\mathbf{W} + L_1 L_2),
\end{equation}
where $L_1 \in \mathbb{R}^{m \times r}$ and $L_2 \in \mathbb{R}^{r \times n}$ denote the low-rank components, typically with $r \ll \min(m,n)$.

In previous works~\cite{he2023efficientdm, feng2025mpqdm}, $L_1$ is initialized using Kaiming Initialization, while $L_2$ is simply initialized to zero~\cite{hu2021lora} with the following characteristics:
\begin{equation}
    L_1L_2 = \mathbf{0}.
\end{equation}
This naive zero-initialization leads to a \emph{cold start} as it does not change any original weights at the beginning of the optimization process. This learns the low rank modules completely from scratch. However, we identify that this cold initialization fails to exploit the prior structural information from the quantization process itself. As a result, the optimization must compensate from scratch for the quantization-induced error, which slows convergence and may miss informative priors.

Quantization inevitably introduces structured error into the weight matrices. Rather than treating this error as noise, we argue it should be viewed as a meaningful prior object that can be approximated and corrected in a principled manner. We propose \textbf{Object-Oriented Low-Rank Initialization (OOLRI)} to efficiently initialize $L_1$ and $L_2$ using quantization error as prior knowledge.

\subsubsection{Quantization Residual Modeling}

We consider the initialization of $L_1$ and $L_2$ as an optimization problem, with the optimization objective being the original weight $\mathbf{W}$. Let $\mathbf{W} \in \mathbb{R}^{m \times n}$ be the full-precision weight matrix, and let $Q(\mathbf{W})$ denote its quantized counterpart. The goal of QA-LoRA is to find a low-rank matrix $L_1 L_2$ that minimizes the quantization loss:
\begin{equation}
    \mathcal{L}(L_1, L_2) = \left\| \mathbf{W} - Q(\mathbf{W} + L_1 L_2) \right\|_F^2.
\end{equation}

Assuming $Q(\cdot)$ is piecewise differentiable or locally smooth, we can  apply a first-order Taylor expansion:
\begin{equation}
    Q(\mathbf{W} + L_1 L_2) \approx Q(\mathbf{W}) + L_1 L_2.
\end{equation}
We then define the residual error form:
\begin{equation}
    \mathbf{E} := \mathbf{W} - Q(\mathbf{W}).
\end{equation}
We can then rewrite the approximate loss:
\begin{equation}
    \mathcal{L}(L_1, L_2) \approx \left\| \mathbf{E} - L_1 L_2 \right\|_F^2.
    \label{eq:approx_loss}
\end{equation}

Hence, the task reduces to computing a low-rank approximation of the quantization residual matrix $\mathbf{E}$.

\begin{figure}[t]
    \centering
    \subfloat[][w/o OOLRI]{
        \includegraphics[width=0.47\linewidth]{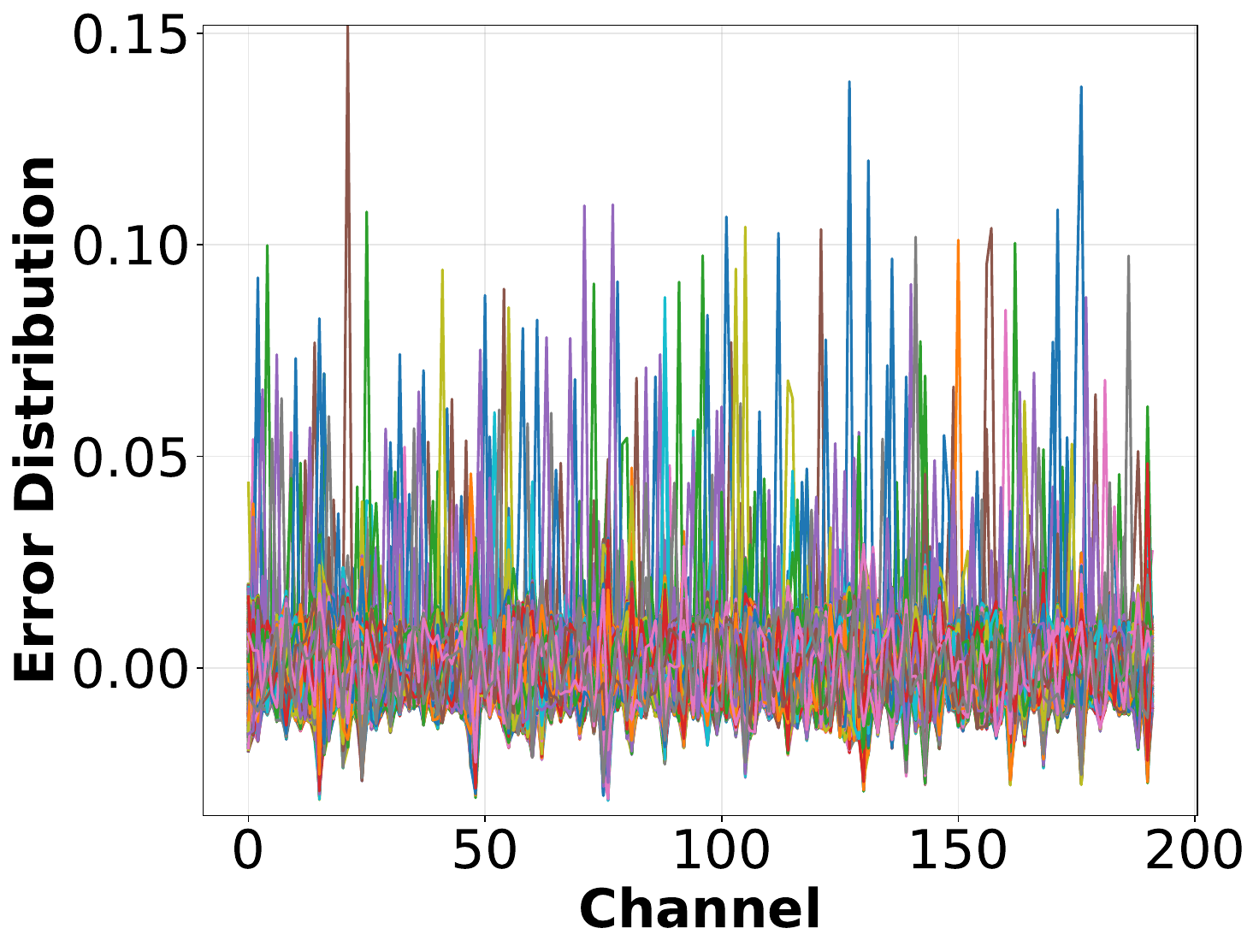}
        \label{}
    }
    \subfloat[][w/ OOLRI]{
        \includegraphics[width=0.47\linewidth]{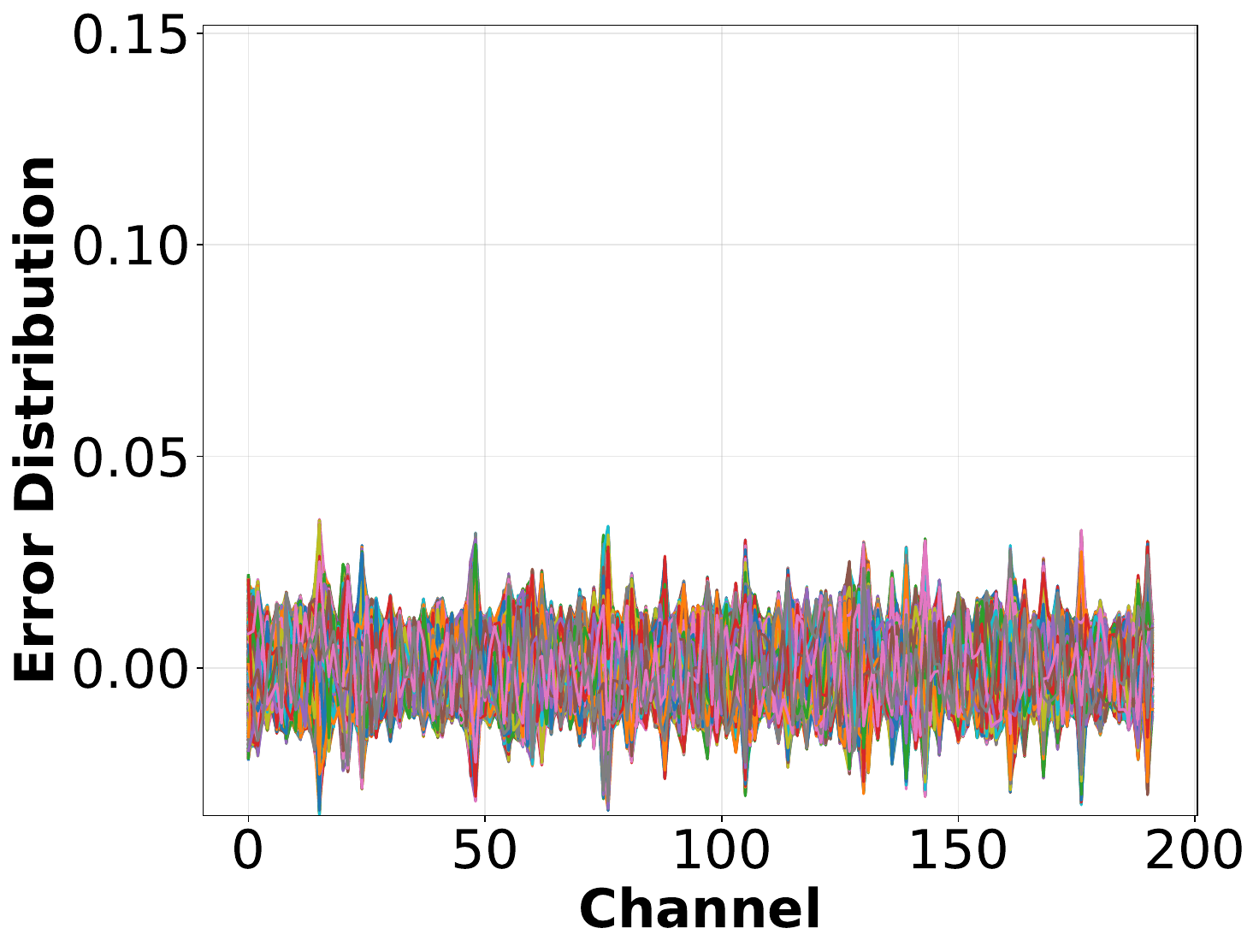}
        \label{}
    }
    \caption{Channel-wise weight error distributions under W3A6 quantization on the LDM-4 model. ``w/o OOLRI'' denotes the zero-initialization without OOLRI technique. OOLRI significantly reduces the overall error for better optimization start.}
    \label{fig:oolri_error}
\end{figure}

\subsubsection{Theoretical Properties of the Modeling Objective}
We claim that our modeling objective Eq.~\eqref{eq:approx_loss} has the following characteristics:
\begin{theorem}
The objective function $f(X) = \left\| \mathbf{E} - X \right\|_F^2$ is convex and $L$-smooth with $L = 2$, where $\mathbf{E} \in \mathbb{R}^{m \times n}$ is the quantization error matrix and $X \in \mathbb{R}^{m \times n}$ is a rank-$r$ matrix.
\label{theorem:smoothness}
\end{theorem}

\begin{proof}
\textbf{Convexity:} For any $X_1, X_2 \in \mathbb{R}^{m \times n}$ and $\lambda \in [0, 1]$, we have:
\begin{align*}
f(\lambda X_1 + (1-\lambda)X_2) &= \left\| \mathbf{E} - \lambda X_1 - (1-\lambda)X_2 \right\|_F^2 \\
&= \left\| \lambda(\mathbf{E} - X_1) + (1-\lambda)(\mathbf{E} - X_2) \right\|_F^2 \\
&\leq \lambda \left\| \mathbf{E} - X_1 \right\|_F^2 + (1-\lambda) \left\| \mathbf{E} - X_2 \right\|_F^2 \\
&= \lambda f(X_1) + (1-\lambda) f(X_2),
\end{align*}
where the inequality follows from the convexity of the squared Frobenius norm.

\textbf{Smoothness:} The gradient of $f(X)$ is $\nabla f(X) = -2(\mathbf{E} - X)$. For any $X_1, X_2 \in \mathbb{R}^{m \times n}$, we have:
\begin{align*}
\left\| \nabla f(X_1) - \nabla f(X_2) \right\|_F &= \left\| -2(X_1 - X_2) \right\|_F \\ 
&= 2 \left\| X_1 - X_2 \right\|_F,
\end{align*}
which satisfies the Lipschitz condition with $L = 2$.
\end{proof}

Theorem~\ref{theorem:smoothness} ensures that our modeling objective $f(L_1L_2)=\left\| \mathbf{E} - L_1L_2 \right\|_F^2$ is convex and L2 smooth, which indicates that there is a global optimal solution and is linearly convergent.

\subsubsection{Low-Rank Initialization via Truncated SVD}

According to Eckart-Young Mirsky theorem~\cite{golub1987generalization}, the optimal rank-$r$ approximation of $\mathbf{E}$ of Eq.~\ref{eq:approx_loss} in Frobenius norm is given by its truncated SVD. While Theorem~\ref{theorem:smoothness} ensures that the SVD approximation is optimal. Let the SVD of $\mathbf{E}$ be:
\begin{equation}
    \mathrm{SVD}(\mathbf{W} - Q(\mathbf{W})) = \mathrm{SVD}(\mathbf{E}) = \sum_{i=1}^{\min(m,n)} s_i \mathbf{u}_i \mathbf{v}_i^\top,
\end{equation}
where $\{s_i\}$ are singular values, and $\{\mathbf{u}_i\}, \{\mathbf{v}_i\}$ are the left and right singular vectors, respectively. Then, the best rank-$r$ approximation is:
\begin{equation}
    \hat{\mathbf{E}}_r = \sum_{i=1}^r s_i \mathbf{u}_i \mathbf{v}_i^\top.
\end{equation}

We therefore initialize the low-rank matrices $L_1$ and $L_2$ as:
\begin{equation}
    L_1 = \left[ \sqrt{s_1} \mathbf{u}_1, \ldots, \sqrt{s_r} \mathbf{u}_r \right], \quad
    L_2 = \left[ \sqrt{s_1} \mathbf{v}_1, \ldots, \sqrt{s_r} \mathbf{v}_r \right]^\top.
\label{eq:oolri_init}
\end{equation}

This initialization captures the dominant modes of error introduced by quantization, providing a semantically meaningful warm start for training. We visualize the quantization error for LDM-4 model in Fig~\ref{fig:oolri_error}. Through OOLRI, the quantization error of weights is greatly reduced to a smaller range, achieving better initialization for further optimization.

\subsection{Overall Framework}

\begin{algorithm}[t]
    \caption{The process of MPQ-DMv2 scheme}
    \label{alg:shceme}
    \renewcommand{\algorithmicrequire}{\textbf{Input:}}
    \renewcommand{\algorithmicensure}{\textbf{Output:}}
    
    \begin{algorithmic}[1]
        \REQUIRE calibration data $X$, pre-trained model $M$ with $N$ layers, training iterations $T$.
        \ENSURE quantized network $Q$.
        
        \STATE  Forward propagate $M(X)$ and gather activations.
            
        \FOR{each layer $i \in [1, L]$}
        \STATE Applying FZ-RMQ for layer weight $W_i$ using Eq.~\eqref{eq:quant_search_objective} for mixed precision quantization;
        \STATE Applying OOLRI for layer weight $W_i$ using Eq.~\eqref{eq:oolri_init} for LoRA initialization;
        \ENDFOR
        
        \FOR{all $t \in [1, T]$}
        \STATE Forward propagate $M(X)$;
        \STATE Back propagate using Eq.~\eqref{eq:total_loss} and update $Q$;
        \ENDFOR
        
        \RETURN calibrated quantization model $Q$.
    \end{algorithmic}
\end{algorithm}

We integrate our proposed MPQ-DMv2 scheme into previous work EfficientDM~\cite{he2023efficientdm} and MPQ-DM~\cite{feng2025mpqdm}. For the quantization initialization phase, we use a batch of calibration data to gather the activations corresponding to different weight layers of the model. We then apply \textit{Flexible Z-Order Residual Mixed Quantization} to allocate quantization bits for different channels and optimize strategies using Eq.~\eqref{eq:quant_search_objective}. After the quantization architecture decision, we use \textit{Object-Oriented Low-Rank Initialization} to perform quantization prior injection by using Eq.~\eqref{eq:oolri_init} to initialize the LoRA module. After initialization is completed, we apply \textit{Memory-based Temporal Relation Distillation} to optimize the quantization parameters and LoRA module. The complete pipeline of MPQ-DMv2 is provided in Algorithm~\ref{alg:shceme}.

\section{Experiment}

\begin{table*}[!htb]
    \caption{Class-conditional image generation results of LDM-4 model on ImageNet 256$\times$256. \textbf{``FP'' denotes the full-precision model. ``$^+$'' denotes allocating an additional 10\% channels for 2-bit.} Best results are in \textbf{bold}.}
    \centering
    \setlength{\tabcolsep}{3.2mm}
    \begin{tabular}{clcccccc}
    \hline
        Task & Method & \makecell{Bit (W/A)} & \makecell{Size (MB)} & IS $\uparrow$ & FID $\downarrow$ 
        & sFID $\downarrow$ & \makecell{Precision $\uparrow$ (\%)} \\ \hline
        \multirow{30}{*}{\makecell{ImageNet \\ 256$\times$256 \\ \\ LDM-4 \\ steps=20 \\ eta=0.0 \\ scale=3.0}} & FP~\cite{rombach2022ldm} & 32/32 & 1529.7 & 364.73 & 11.28 & 7.70 & 93.66 \\ \cline{2-8}
        & PTQ-D~\cite{he2024ptqd} & 3/6 & 144.5 & 162.90 & 17.98 & 57.31 & 63.13 \\
        & TFMQ~\cite{huang2024tfmq} & 3/6 & 144.5 & 174.31 & 15.90 & 40.63 & 67.42 \\
        & QuEST~\cite{wang2024quest} & 3/6 & 144.6 & 194.32 & 14.32 & 31.87 & 72.80 \\
        & EfficientDM~\cite{he2023efficientdm} & 3/6 & 144.6 & 299.63 & 7.23 & 8.18 & 86.27 \\
        & HAWQ-V3~\cite{yao2021hawqv3} & 3/6 & 144.6 & 303.79 & 6.94 & 8.01 & 87.76 \\
        & MPQ-DM~\cite{feng2025mpqdm} & 3/6 & 144.6 & 306.33 & 6.67 & 7.93 & 88.65 \\ 
        & \cellcolor[gray]{0.9}\textbf{MPQ-DMv2} & \cellcolor[gray]{0.9}3/6 & \cellcolor[gray]{0.9}144.6 & \cellcolor[gray]{0.9}\textbf{313.38} & \cellcolor[gray]{0.9}\textbf{6.65} & \cellcolor[gray]{0.9}\textbf{7.91} & \cellcolor[gray]{0.9}\textbf{89.74} \\ \cline{2-8}
        & PTQ-D~\cite{he2024ptqd} & 3/4 & 144.5 & 10.86 & 286.57 & 273.16 & 0.02 \\
        & TFMQ~\cite{huang2024tfmq} & 3/4 & 144.5 & 13.08 & 223.51 & 256.32 & 0.04 \\
        & QuEST~\cite{wang2024quest} & 3/4 & 175.5 & 15.22  & 202.44 & 253.64 & 0.04 \\
        & EfficientDM~\cite{he2023efficientdm} & 3/4 & 144.6 & 134.30  & 11.02 & 9.52 & 70.52 \\
        & HAWQ-V3~\cite{yao2021hawqv3} & 3/4 & 144.6 & 152.61 & 8.49 & 9.26 & 75.02 \\
        & MPQ-DM~\cite{feng2025mpqdm} & 3/4 & 144.6 & 197.43 & 6.72 & 9.02 & 81.26 \\ 
        & \cellcolor[gray]{0.9}\textbf{MPQ-DMv2} & \cellcolor[gray]{0.9}3/4 & \cellcolor[gray]{0.9}144.6 & \cellcolor[gray]{0.9}\textbf{197.88} & \cellcolor[gray]{0.9}\textbf{6.60} & \cellcolor[gray]{0.9}\textbf{8.64} & \cellcolor[gray]{0.9}\textbf{81.50} \\ \cline{2-8}
        & PTQ-D~\cite{he2024ptqd} & 2/6 & 96.7 & 70.43 & 40.29 & 35.70 & 43.79 \\
        & TFMQ~\cite{huang2024tfmq} & 2/6 & 96.7 & 77.26 & 36.22 & 33.05 & 45.88 \\
        & QuEST~\cite{wang2024quest} & 2/6 & 96.8 & 86.83 & 32.37 & 31.58 & 47.74 \\
        & EfficientDM~\cite{he2023efficientdm} & 2/6 & 96.8 & 69.64 & 29.15 & 12.94 & 54.70 \\
        & HAWQ-V3~\cite{yao2021hawqv3} & 2/6 & 96.8 & 88.25 & 22.73 & 11.68 & 57.04 \\
        & MPQ-DM~\cite{feng2025mpqdm} & 2/6 & 96.8 & 102.51 & 15.89 & 10.54 & 67.74 \\
        & MPQ-DM$^+$~\cite{feng2025mpqdm} & 2/6 & 101.6 & 136.35 & 11.00 & 9.41 & 72.84 \\
        & \cellcolor[gray]{0.9}MPQ-DMv2 & \cellcolor[gray]{0.9}2/6 & \cellcolor[gray]{0.9}96.8 & \cellcolor[gray]{0.9}116.61 & \cellcolor[gray]{0.9}13.91 & \cellcolor[gray]{0.9}10.50 & \cellcolor[gray]{0.9}69.56 \\
        & \cellcolor[gray]{0.9}\textbf{MPQ-DMv2}$^+$ & \cellcolor[gray]{0.9}2/6 & \cellcolor[gray]{0.9}101.6 & \cellcolor[gray]{0.9}\textbf{136.73} & \cellcolor[gray]{0.9}\textbf{10.96} & \cellcolor[gray]{0.9}\textbf{9.37} & \cellcolor[gray]{0.9}\textbf{73.05} \\ \cline{2-8}
        & PTQ-D~\cite{he2024ptqd} & 2/4 & 96.7 & 9.25 & 336.57 & 288.42 & 0.01 \\
        & TFMQ~\cite{huang2024tfmq} & 2/4 & 96.7 & 12.76 & 300.03 & 272.64 & 0.03 \\
        & QuEST~\cite{wang2024quest} & 2/4 & 127.7 & 14.09 & 285.42 & 270.12 & 0.03 \\
        & EfficientDM~\cite{he2023efficientdm} & 2/4 & 96.8 & 25.20 & 64.45 & 14.99 & 36.63 \\
        & HAWQ-V3~\cite{yao2021hawqv3} & 2/4 & 96.8 & 33.21 & 52.63 & 14.00 & 42.95 \\
        & MPQ-DM~\cite{feng2025mpqdm} & 2/4 & 96.8 & 43.95 & 36.59 & 12.20 & 52.14 \\
        & MPQ-DM$^{+}$~\cite{feng2025mpqdm} & 2/4 & 101.6 & 60.55 & 27.11 & 11.47 & 57.84 \\
        & \cellcolor[gray]{0.9}MPQ-DMv2 & \cellcolor[gray]{0.9}2/4 & \cellcolor[gray]{0.9}96.8 & \cellcolor[gray]{0.9}49.79 & \cellcolor[gray]{0.9}32.55 & \cellcolor[gray]{0.9}12.12 &\cellcolor[gray]{0.9}54.40 \\ 
        & \cellcolor[gray]{0.9}\textbf{MPQ-DMv2}$^{+}$ & \cellcolor[gray]{0.9}2/4 & \cellcolor[gray]{0.9}101.6 & \cellcolor[gray]{0.9}\textbf{60.73} & \cellcolor[gray]{0.9}\textbf{26.98} & \cellcolor[gray]{0.9}\textbf{11.26} & \cellcolor[gray]{0.9}\textbf{58.16} \\ \hline
    \end{tabular}
    \label{tab:imagenet}
\end{table*}

\begin{table*}[!htb]
    \caption{Unconditional image generation results of LDM models
 on LSUN datasets. }
    \centering
    \setlength{\tabcolsep}{2.8mm}
    \begin{tabular}{clcccccc}
    \hline
        Task & Method & \makecell{Bit (W/A)} & \makecell{Size (MB)} & FID $\downarrow$ & sFID $\downarrow$ & \makecell{Precision $\uparrow$  (\%)} \\ \hline
        \multirow{27}{*}{\makecell{LSUN-Bedrooms \\ 256$\times$256 \\ \\ LDM-4 \\ steps=100 \\ eta=1.0}} & FP~\cite{rombach2022ldm} & 32/32 & 1045.4 & 7.39 & 12.18 & 52.04 \\ \cline{2-7}
        & PTQ-D~\cite{he2024ptqd} & 3/6 & 98.3 & 113.42 & 43.85 & 10.06\\
        & TFMQ~\cite{huang2024tfmq} & 3/6 & 98.3 & 26.42 & 30.87 & 38.29 \\
        & QuEST~\cite{wang2024quest} & 3/6 & 98.4 & 21.03 & 28.75 & 40.32 \\
        & EfficientDM~\cite{he2023efficientdm} & 3/6 & 98.4 &   13.37 &   16.14 &   44.55 \\
        & MPQ-DM~\cite{feng2025mpqdm} & 3/6 & 98.4 & 11.58 & 15.44 & 47.13 \\
        & \cellcolor[gray]{0.9}\textbf{MPQ-DMv2} & \cellcolor[gray]{0.9}3/6 & \cellcolor[gray]{0.9}98.4 & \cellcolor[gray]{0.9}\textbf{10.72} & \cellcolor[gray]{0.9}\textbf{15.14} & \cellcolor[gray]{0.9}\textbf{49.15} \\ \cline{2-7}
        & PTQ-D~\cite{he2024ptqd} & 3/4 & 98.3 & 100.07 & 50.29 & 11.64 \\
        & TFMQ~\cite{huang2024tfmq} & 3/4 & 98.3 & 25.74 & 35.18 & 32.20 \\
        & QuEST~\cite{wang2024quest} & 3/4 & 110.4 &   19.08 & 32.75 &   40.64 \\
        & EfficientDM~\cite{he2023efficientdm} & 3/4 & 98.4 & 20.39 &   20.65 & 38.70 \\
        & MPQ-DM~\cite{feng2025mpqdm} & 3/4 & 98.4 & 14.80 & 16.72 & 43.61 \\
        & \cellcolor[gray]{0.9}\textbf{MPQ-DMv2} & \cellcolor[gray]{0.9}3/4 & \cellcolor[gray]{0.9}98.4 & \cellcolor[gray]{0.9}\textbf{14.64} & \cellcolor[gray]{0.9}\textbf{16.70} & \cellcolor[gray]{0.9}\textbf{44.50} \\ \cline{2-7}
        & PTQ-D~\cite{he2024ptqd} & 2/6 & 65.7 & 86.65 & 53.52 & 10.27 \\
        & TFMQ~\cite{huang2024tfmq} & 2/6 & 65.7 & 28.72 & 29.02 & 34.57 \\
        & QuEST~\cite{wang2024quest} & 2/6 & 65.7 & 29.64 & 29.73 & 34.55 \\
        & EfficientDM~\cite{he2023efficientdm} & 2/6 & 65.7 & 25.07 & 22.17 & 34.59 \\
        & MPQ-DM~\cite{feng2025mpqdm} & 2/6 & 65.7 & 17.12 & 19.06 & 40.90 \\
        & MPQ-DM$^+$~\cite{feng2025mpqdm} & 2/6 & 68.9 & 16.54 & 18.36 & 41.80 \\
        & \cellcolor[gray]{0.9}MPQ-DMv2 & \cellcolor[gray]{0.9}2/6 & \cellcolor[gray]{0.9}65.7 & \cellcolor[gray]{0.9}17.09 & \cellcolor[gray]{0.9}18.58 & \cellcolor[gray]{0.9}41.81 \\
        & \cellcolor[gray]{0.9}\textbf{MPQ-DMv2}$^+$ & \cellcolor[gray]{0.9}2/6 & \cellcolor[gray]{0.9}68.9 & \cellcolor[gray]{0.9}\textbf{16.18} & \cellcolor[gray]{0.9}\textbf{18.32} & \cellcolor[gray]{0.9}\textbf{44.13} \\ \cline{2-7}
        & PTQ-D~\cite{he2024ptqd} & 2/4 & 65.7 & 147.25 & 49.97 & 9.26 \\
        & TFMQ~\cite{huang2024tfmq} & 2/4 & 65.7 & 25.77 & 36.74 & 32.86 \\
        & QuEST~\cite{wang2024quest} & 2/4 & 77.7 & 24.92 & 36.33 & 32.82 \\
        & EfficientDM~\cite{he2023efficientdm} & 2/4 & 65.7 & 33.09 & 25.54 & 28.42 \\
        & MPQ-DM~\cite{feng2025mpqdm} & 2/4 & 65.7 & 21.69 & 21.58 & 38.69 \\
        & MPQ-DM$^+$~\cite{feng2025mpqdm} & 2/4 & 68.9 & 20.28 & 19.42 & 38.92 \\
        & \cellcolor[gray]{0.9}MPQ-DMv2 & \cellcolor[gray]{0.9}2/4 & \cellcolor[gray]{0.9}65.7 & \cellcolor[gray]{0.9}21.65 & \cellcolor[gray]{0.9}19.95 & \cellcolor[gray]{0.9}38.70 \\
        & \cellcolor[gray]{0.9}\textbf{MPQ-DMv2}$^+$ & \cellcolor[gray]{0.9}2/4 & \cellcolor[gray]{0.9}68.9 & \cellcolor[gray]{0.9}\textbf{20.15} & \cellcolor[gray]{0.9}\textbf{19.37} & \cellcolor[gray]{0.9}\textbf{39.91} \\ \hline
        \multirow{27}{*}{\makecell{LSUN-Churches \\ 256$\times$256 \\ \\ LDM-8 \\ steps=100 \\ eta=0.0}} & FP~\cite{rombach2022ldm} & 32/32 & 1125.2 & 5.55 & 10.75 & 67.43 \\ \cline{2-7}
        & PTQ-D~\cite{he2024ptqd} & 3/6 & 106.0 & 59.43 & 40.26 & 13.37 \\
        & TFMQ~\cite{huang2024tfmq} & 3/6 & 106.0 & 13.53 & 22.10 & 62.74 \\
        & QuEST~\cite{wang2024quest} & 3/6 & 106.1 & 22.19 & 32.79 & 60.73 \\
        & EfficientDM~\cite{he2023efficientdm} & 3/6 & 106.1 &   9.53 &   13.70 &   62.92 \\
        & MPQ-DM~\cite{feng2025mpqdm} & 3/6 & 106.1 & 9.28 & 13.37 & 63.73 \\
        & \cellcolor[gray]{0.9}\textbf{MPQ-DMv2} & \cellcolor[gray]{0.9}3/6 & \cellcolor[gray]{0.9}106.1 & \cellcolor[gray]{0.9}\textbf{8.65} & \cellcolor[gray]{0.9}\textbf{12.88} & \cellcolor[gray]{0.9}\textbf{64.08} \\ \cline{2-7}
        & PTQ-D~\cite{he2024ptqd} & 3/4 & 106.0 & 77.08 & 49.63 & 10.25 \\
        & TFMQ~\cite{huang2024tfmq} & 3/4 & 106.0 & 35.51 & 48.59 & 55.32 \\
        & QuEST~\cite{wang2024quest} & 3/4 & 122.4 & 40.74 & 53.63 & 52.78 \\
        & EfficientDM~\cite{he2023efficientdm} & 3/4 & 106.1 &   15.59 &   18.16 &   57.92 \\
        & MPQ-DM~\cite{feng2025mpqdm} & 3/4 & 106.1 & 14.08 & 16.91 & 59.68 \\
        & \cellcolor[gray]{0.9}\textbf{MPQ-DMv2} & \cellcolor[gray]{0.9}3/4 & \cellcolor[gray]{0.9}106.1 & \cellcolor[gray]{0.9}\textbf{10.58} & \cellcolor[gray]{0.9}\textbf{14.60} & \cellcolor[gray]{0.9}\textbf{62.87} \\ \cline{2-7}
        & PTQ-D~\cite{he2024ptqd} & 2/6 & 70.9 & 63.38 & 46.63 & 12.14 \\
        & TFMQ~\cite{huang2024tfmq} & 2/6 & 70.9 & 25.51 & 35.83 & 54.75 \\
        & QuEST~\cite{wang2024quest} & 2/6 & 70.9 & 23.03 & 35.13 & 56.90 \\
        & EfficientDM~\cite{he2023efficientdm} & 2/6 & 70.9 & 16.98 & 18.18 & 57.39 \\
        & MPQ-DM~\cite{feng2025mpqdm} & 2/6 & 70.9 & 15.61 & 17.44 & 59.03 \\
        & MPQ-DM$^+$~\cite{feng2025mpqdm} & 2/6 & 74.4 & 13.38 & 15.59 & 61.00 \\
        & \cellcolor[gray]{0.9}MPQ-DMv2 & \cellcolor[gray]{0.9}2/6 & \cellcolor[gray]{0.9}70.9 & \cellcolor[gray]{0.9}12.83 & \cellcolor[gray]{0.9}15.76 & \cellcolor[gray]{0.9}60.88 \\
        & \cellcolor[gray]{0.9}\textbf{MPQ-DMv2}$^+$ & \cellcolor[gray]{0.9}2/6 & \cellcolor[gray]{0.9}74.4 & \cellcolor[gray]{0.9}\textbf{12.23} & \cellcolor[gray]{0.9}\textbf{15.56} & \cellcolor[gray]{0.9}\textbf{61.91} \\ \cline{2-7}
        & PTQ-D~\cite{he2024ptqd} & 2/4 & 70.9 & 81.95 & 50.66 & 9.47 \\
        & TFMQ~\cite{huang2024tfmq} & 2/4 & 70.9 & 51.44 & 64.07 & 42.25 \\
        & QuEST~\cite{wang2024quest} & 2/4 & 86.9 & 50.53 & 63.33 & 45.86 \\
        & EfficientDM~\cite{he2023efficientdm} & 2/4 & 70.9 & 22.74 & 22.55 & 53.00 \\
        & MPQ-DM~\cite{feng2025mpqdm} & 2/4 & 70.9 & 21.83 & 21.38 & 53.99 \\
        & MPQ-DM$^+$~\cite{feng2025mpqdm} & 2/4 & 74.4 & 16.91 & 18.57 & 58.04 \\
        & \cellcolor[gray]{0.9}MPQ-DMv2 & \cellcolor[gray]{0.9}2/4 & \cellcolor[gray]{0.9}70.9 & \cellcolor[gray]{0.9}15.50 & \cellcolor[gray]{0.9}\textbf{17.81} & \cellcolor[gray]{0.9}59.95 \\
        & \cellcolor[gray]{0.9}\textbf{MPQ-DMv2}$^+$ & \cellcolor[gray]{0.9}2/4 & \cellcolor[gray]{0.9}74.4 & \cellcolor[gray]{0.9}\textbf{15.20} & \cellcolor[gray]{0.9}18.02 & \cellcolor[gray]{0.9}\textbf{60.39} \\ \hline
    \end{tabular}
    \label{tab:lsun_bedroom}
\end{table*}

\subsection{Models and Datasets}
We perform comprehensive experiments that include unconditional image generation, class-conditional image generation, and text-conditional image generation tasks on two Unet-based diffusion models: latent-space diffusion model (LDM) and Stable Diffusion v1.4~\cite{rombach2022ldm}. For LDM, our investigations spanned multiple datasets, including LSUN-Bedrooms, LSUN-Churches~\cite{yu2015lsun}, and ImageNet~\cite{deng2009imagenet}, all with a resolution of 256×256. Furthermore, we employ Stable Diffusion for text-conditional image generation on randomly sampled 10k COCO2014~\cite{lin2014mscoco} validation set prompts with a resolution of 512×512. We further extend the experiment on Diffusion Transformer (DiT) architecture. We use DiT-XL/2 model~\cite{peebles2023dit} on ImageNet dataset under 256 and 512 resolutions for evaluation. This diverse set of experiments, conducted on different models, datasets, and tasks, allows us to validate the effectiveness of our MPQ-DMv2 comprehensively. 

\subsection{Implementation Details}
Same with previous work~\cite{feng2025mpqdm}, we allocate an additional 10\% number of channels for 2-bit during the search process of FZ-RMQ, named MPQ-DMv2$^+$. This results in only a 0.6\% increase in model size compared with FP model. We compare our MPQ-DMv2 with baseline method EfficientDM~\cite{he2023efficientdm}, layer-wise mixed precision HAWQ-v3 \cite{yao2021hawqv3}, channel-wise mixed precision MPQ-DM~\cite{feng2025mpqdm}, and other PTQ-based methods PTQ-D\cite{he2024ptqd}, TFMQ~\cite{huang2024tfmq}, QuEST~\cite{wang2024quest} which possess similar time consumption. For DiT models, we compare with baseline method PTQ4DiT~\cite{wu2024ptq4dit}. We do not report the results of Q-Diffusion~\cite{li2023qdiffusion} and PTQ4DM~\cite{shang2023ptq4dm} as they face severe performance degradation under low-bit quantization, and totally cannot generate meaningful images as previous works~\cite{he2024ptqd, he2023efficientdm} reported. We follow previous works~\cite{he2023efficientdm, feng2025mpqdm} to perform quantization-aware low-rank fine-tuning for quantization diffusion models. For LDM models training process, we fine-tune LoRA weights and quantization parameters for 16K iterations with a batchsize of 4. For Stable Diffusion training process, we fine-tune parameters for 30K iterations with a batchsize of 2. For Dit-XL/2 models, we fine-tune 40K iterations with a batchsize of 32 for 256 resolution and 80k iterations with a batchsize of 8 for 512 resolution. All the learning rates and calibration data are consistent with existing work MPQ-DM~\cite{feng2025mpqdm}. For MTRD, we set $L=20k$ and $k=1024$ for ImageNet dataset; $L=4K$ and $k=204$ for LSUN datasets and COCO dataset.

\subsection{Evaluation Details}
For the quantization experiment bit setting selection, we fully follow and cover the existing work MPQ-DM in the experiment models and bit settings. In addition, we conducted additional experiments on Stable Diffusion with extra bit setting and on new DiT models. We use IS~\cite{salimans2016is_metric}, FID \cite{heusel2017fid}, sFID \cite{nash2021sfid}, and Precision to evaluate LDM and DiT performance. We only report IS for the ImageNet dataset. Because the Inception Score is not a reasonable metric for datasets that have significantly different domains and categories from ImageNet. These metrics are all evaluated using ADM’s TensorFlow evaluation suite. For Stable Diffusion, we use CLIP Score \cite{hessel2021clipscore} for evaluation. We sample 50k samples for LDM model, 10k samples for Stable Diffusion model and DiT-XL/2 256$\times$256 model, 5k samples for DiT-XL/2 512$\times$512 model, following prior works~\cite{he2023efficientdm, wang2024quest, feng2025mpqdm, wu2024ptq4dit}.

\subsection{Experiment Results}

\textbf{Class-conditional Generation.}
We conduct class-conditional generation experiment on ImageNet 256$\times$256 dataset~\cite{deng2009imagenet}, focusing on LDM-4~\cite{rombach2022ldm}. We present the results in Table~\ref{tab:imagenet}. Our MPQ-DMv2 outperforms existing methods in all extreme-low bit settings, demonstrating the effectiveness and generalization of our approach. Compared with the main baseline method MPQ-DM, our MPQ-DMv2 has improved to varying degrees in all quantization settings, indicating that the improvements our MPQ-DMv2 are effective for low-bit quantization. Under the W2A6 setting, our MPQ-DMv2 achieved FID surpassing FP method for the first time and improved IS by more than 10. In the most difficult W2A4 setting, MPQ-DMv2 achieved significant improvement, reducing FID from 36.59 to 32.55, demonstrating the potential of our method in extremely low bit settings. For MPQ-DMv2$^+$, which is designed to handle extreme settings of 2 bits, it achieved an almost 6 decrease in FID and only increased the parameter count by less than 5MB.

\begin{table*}[!htp]
    \centering
    \caption{Class-conditional Image generation results of DiT-XL/2 models on ImageNet datasets. 
    } 
    \setlength{\tabcolsep}{2.8mm}
    \begin{tabular}{clcccccc}
    \hline
        Task & Method & \makecell{Bit (W/A)} & \makecell{Size (MB)} & IS $\uparrow$ & FID $\downarrow$ 
        & sFID $\downarrow$ & \makecell{Precision $\uparrow$ (\%)} \\ \hline 
        \multirow{10}{*}{\makecell{ImageNet \\ 256$\times$256 \\ \\ DiT-XL/2 \\ 50steps \\ scale=1.5}} & FP~\cite{peebles2023dit} & 32/32 & 2575.42 & 246.24 & 6.02 & 21.77 & 78.12 \\ \cline{2-8}
        & PTQ4DiT~\cite{wu2024ptq4dit} & 4/6 & 323.8 & 66.32 & 25.70 & 37.05 & 59.45 \\
        & MPQ-DM~\cite{feng2025mpqdm}  & 4/6 & 323.8 & 88.54 & 19.98 & 32.18 & 71.60 \\
        & \cellcolor[gray]{0.9}\textbf{MPQ-DMv2} & \cellcolor[gray]{0.9}4/6 & \cellcolor[gray]{0.9}323.8 & \cellcolor[gray]{0.9}\textbf{97.03} & \cellcolor[gray]{0.9}\textbf{17.35} & \cellcolor[gray]{0.9}\textbf{29.34} & \cellcolor[gray]{0.9}\textbf{72.51} \\ \cline{2-8}
        & PTQ4DiT~\cite{wu2024ptq4dit} & 3/8 & 243.3 & 9.07 & 130.34 & 64.88 & 12.40 \\
        & MPQ-DM~\cite{feng2025mpqdm}  & 3/8 & 243.3 & 54.73 & 39.64 & 41.27 & 60.88 \\
        & \cellcolor[gray]{0.9}\textbf{MPQ-DMv2} & \cellcolor[gray]{0.9}3/8 & \cellcolor[gray]{0.9}243.3 & \cellcolor[gray]{0.9}\textbf{60.20} & \cellcolor[gray]{0.9}\textbf{35.64} & \cellcolor[gray]{0.9}\textbf{38.62} & \cellcolor[gray]{0.9}\textbf{62.58} \\ \cline{2-8}
        & PTQ4DiT~\cite{wu2024ptq4dit} & 3/6 & 243.3 & 6.79 & 150.74 & 68.94 & 12.28 \\
        & MPQ-DM~\cite{feng2025mpqdm}  & 3/6 & 243.3 & 44.02 & 48.69 & 55.80 & 55.34 \\
        & \cellcolor[gray]{0.9}\textbf{MPQ-DMv2} & \cellcolor[gray]{0.9}3/6 & \cellcolor[gray]{0.9}243.3 & \cellcolor[gray]{0.9}\textbf{46.34} & \cellcolor[gray]{0.9}\textbf{45.26}& \cellcolor[gray]{0.9}\textbf{49.83} & \cellcolor[gray]{0.9}\textbf{56.84} \\ \hline
        \multirow{7}{*}{\makecell{ImageNet \\ 512$\times$512 \\ \\ DiT-XL/2 \\ 50steps \\ scale=1.5}} & FP~\cite{peebles2023dit} & 32/32 & 2575.42 & 213.86 & 11.28 & 41.70 & 81.00 \\ \cline{2-8}
        & PTQ4DiT~\cite{wu2024ptq4dit} & 4/6 & 323.8 & 7.30 & 169.89 & 68.61 & 7.16 \\
        & MPQ-DM~\cite{feng2025mpqdm}  & 4/6 & 323.8 & 58.91 & 44.68 & 56.13 & 60.88 \\
        & \cellcolor[gray]{0.9}\textbf{MPQ-DMv2} & \cellcolor[gray]{0.9}4/6 & \cellcolor[gray]{0.9}323.8 & \cellcolor[gray]{0.9}\textbf{62.68} & \cellcolor[gray]{0.9}\textbf{42.88} & \cellcolor[gray]{0.9}\textbf{50.45} & \cellcolor[gray]{0.9}\textbf{60.88} \\ \cline{2-8}
        & PTQ4DiT~\cite{wu2024ptq4dit} & 3/8 & 243.3 & 4.45 & 247.55 & 123.70 & 3.90 \\
        & MPQ-DM~\cite{feng2025mpqdm} & 3/8 & 243.3 & 38.21 & 60.96 & 58.39 & 49.44 \\
        & \cellcolor[gray]{0.9}\textbf{MPQ-DMv2} & \cellcolor[gray]{0.9}3/8 & \cellcolor[gray]{0.9}243.3 & \cellcolor[gray]{0.9}\textbf{45.49} & \cellcolor[gray]{0.9}\textbf{53.63} & \cellcolor[gray]{0.9}\textbf{53.64} & \cellcolor[gray]{0.9}\textbf{54.32} \\ \hline
    \end{tabular}
    \label{tab:dit}
\end{table*}

\textbf{Unconditional Generation.}
We conduct unconditional generation experiment on LSUN-Bedrooms dataset over LDM-4 and LSUN-Churches dataset~\cite{yu2015lsun} over LDM-8~\cite{rombach2022ldm} with 256$\times$256 resolution. We present all the experiment results in Table~\ref{tab:lsun_bedroom}. MPQ-DMv2 still outperforms all other existing methods under all bit settings. For LSUN-Churches dataset, MPQ-DMv2 has made significant progress in almost all settings. Under W3A4 setting, the FID was reduced from 14.08 to 10.58, achieving a nearly 30\% decrease. Under W2A6 setting, sFID was reduced from 17.44 to 15.76, and the improvement even exceeded the improvement of EfficientDM by MPQ-DM. Under the most challenging W2A4 setting, we further reduced the FID from 21.83 to 15.50, achieving a FID decrease of over 6. It is worth mentioning that our MPQ-DMv2 even surpasses the performance of MPQ-DM$^+$ under W2A4 setting, further demonstrating the significant performance improvement of our method at ultra-low bit widths.

\textbf{Text-to-image Generation.}
To verify the performance on more detailed and difficult text to image generation task, we conduct experiment on Stable Diffusion model~\cite{rombach2022ldm} with 512$\times$512 resolution. We use randomly selected 10k COCO2014 validation set prompts~\cite{lin2014mscoco} to validate the model's text-to-image capabilities. We present the evaluation results in Table~\ref{tab:sd}. In high-resolution text-to-image generation tasks, MPQ-DMv2 still outperforms all existing methods in various bit settings. In W3A6 and W2A6 settings, we achieve over 0.3 CLIP Score improvement. Under the W3A4 setting, we even achieved a score improvement of over 1.3, demonstrating the effectiveness of our method.

\textbf{DiT architecture experiment.}
To further validate the generality and effectiveness of our MPQ-DMv2 on different diffusion model architectures, we conducted comprehensive experiments on the advanced DiT-XL/2 model~\cite{peebles2023dit}. To verify the generality of the method, we conducted experiments at two different high resolutions, ImageNet 256$\times$266 and 512$\times$512~\cite{deng2009imagenet}. For the DiT architecture, we chose PTQ4DiT~\cite{wu2024ptq4dit} as the baseline method for comparison and comprehensively compared the original MPQ-DM~\cite{feng2025mpqdm} method as the state-of-the-art comparison method. We prensent the evaluation results in Tab.~\ref{tab:dit}. It can be seen that under various low-bit quantization settings, the mixed precision quantization framework MPQ-DM is significantly better than PTQ4DiT, and our improved MPQ-DMv2 has achieved further improvement and SOTA under various settings. Under the W3A8 setting with a resolution of 256, we reduced the FID from 39.64 to 35.64 and significantly surpassed the 130.34 of the PTQ4DiT method. Under the more challenging 512 resolution W3A8 setting, we reduced the FID from 60.96 to 53.63, achieving a FID reduction of over 7 and significantly better than PTQ4DiT's 247.55. This reflects the significant improvement of MPQ-DMv2 over the original method MPQ-DM, and is significantly better than the existing baseline method PTQ4DiT at low-bit widths.

\begin{table}[!htp]
    \centering
    \caption{Text-to-image generation results of Stable Diffusion v1.4 using 10k COCO2014 validation set prompts. 
    } 
    \setlength{\tabcolsep}{1.6mm}
    \begin{tabular}{clccc}
    \hline
        Task & Method & \makecell{Bit \\ (W/A)} & \makecell{Size \\ (MB)} & CLIP Score $\uparrow$ \\ \hline
        \multirow{15}{*}{\makecell{MS-COCO \\ 512$\times$512 \\ \\ Stable-Diffusion \\ steps=50 \\ eta=0.0 \\ scale=7.5}}& FP~\cite{rombach2022ldm} & 32/32 & 3279.1 & 31.25 \\ \cline{2-5}
        & QuEST~\cite{wang2024quest} & 4/6 & 410.4 & 30.16 \\
        & EfficientDM~\cite{he2023efficientdm} & 4/6 & 410.4 & 30.24 \\
        & MPQ-DM~\cite{feng2025mpqdm} & 4/6 & 410.4 & 30.59 \\
        & \cellcolor[gray]{0.9}\textbf{MPQ-DMv2} & \cellcolor[gray]{0.9}4/6 & \cellcolor[gray]{0.9}410.4 & \cellcolor[gray]{0.9}\textbf{30.71} \\ \cline{2-5}
        & QuEST~\cite{wang2024quest} & 3/6 & 309.8 & 28.76 \\
        & EfficientDM~\cite{he2023efficientdm} & 3/6 & 309.8 & 29.12 \\
        & MPQ-DM~\cite{feng2025mpqdm} & 3/6 & 309.8 & 29.63 \\
        & \cellcolor[gray]{0.9}\textbf{MPQ-DMv2} & \cellcolor[gray]{0.9}3/6 & \cellcolor[gray]{0.9}309.8 & \cellcolor[gray]{0.9}\textbf{29.95} \\ \cline{2-5}
        & QuEST~\cite{wang2024quest} & 3/4 & 332.9 & 26.55 \\
        & EfficientDM~\cite{he2023efficientdm} & 3/4 & 309.8 & 26.63 \\
        & MPQ-DM~\cite{feng2025mpqdm} & 3/4 & 309.8 & 26.96 \\
        & \cellcolor[gray]{0.9}\textbf{MPQ-DMv2} & \cellcolor[gray]{0.9}3/4 & \cellcolor[gray]{0.9}309.8 & \cellcolor[gray]{0.9}\textbf{28.28} \\ \cline{2-5}
        & QuEST~\cite{wang2024quest} & 2/6 & 207.4 & 22.88 \\
        & EfficientDM~\cite{he2023efficientdm} & 2/6 & 207.4 & 22.94 \\
        & MPQ-DM~\cite{feng2025mpqdm} & 2/6 & 207.4 & 23.23 \\
        & MPQ-DM$^+$~\cite{feng2025mpqdm} & 2/6 & 217.6 & 25.02 \\
        & \cellcolor[gray]{0.9}MPQ-DMv2 & \cellcolor[gray]{0.9}2/6 & \cellcolor[gray]{0.9}207.4 & \cellcolor[gray]{0.9}23.96 \\
        & \cellcolor[gray]{0.9}\textbf{MPQ-DMv2}$^+$ & \cellcolor[gray]{0.9}2/6 & \cellcolor[gray]{0.9}217.6 & \cellcolor[gray]{0.9}\textbf{25.06} \\ \hline
    \end{tabular}
    \label{tab:sd}
\end{table}

\subsection{Visual Comparison}

\textbf{Comparison on ImageNet datasets.} We present visual comparison results about LDM-4 ImageNet 256$\times$256 model in Fig.~\ref{fig:visual_imagenet} under W3A6 quantization setting. We mainly compare our MPQ-DMv2 with the current SOTA methods MPQ-DM and EfficientDM. For a comprehensive comparison, we visualized images of different categories. It can be seen that compared to existing methods, our MPQ-DMv2 has better image quality and details, and the overall image is more similar to the image of the full-precision (FP) model. For example, for the pizza image in the first column, our MPQ-DMv2 is significantly better in color and more detailed in material characterization on top of the pizza.

\textbf{Comparison on Text-to-Image tasks.} We further present visual comparison results about Stable Diffusion model of text-to-image generation in Fig.~\ref{fig:visual_sd} under W4A6 quantization setting. We visualized the results of three different complex scenarios for prompt generation. Compared with existing methods, our MPQ-DMv2 has significantly improved image quality and is more similar in style and content to the images generated by the FP model. For the first row, the house details and smoke wheels generated by MPQ-DMv2 are more similar to the FP model. For the second line, the castle generated by MPQ-DMv2 has a style that is significantly closer to the FP model. As a comparison, the image quality generated by other methods has significantly decreased. For the third line, the towers generated by other methods have serious content discrepancies with the FP model, while MPQ-DMv2 has a clear preservation of the original style.

\begin{figure*}[t]
    \centering
    \includegraphics[width=1.0\linewidth]{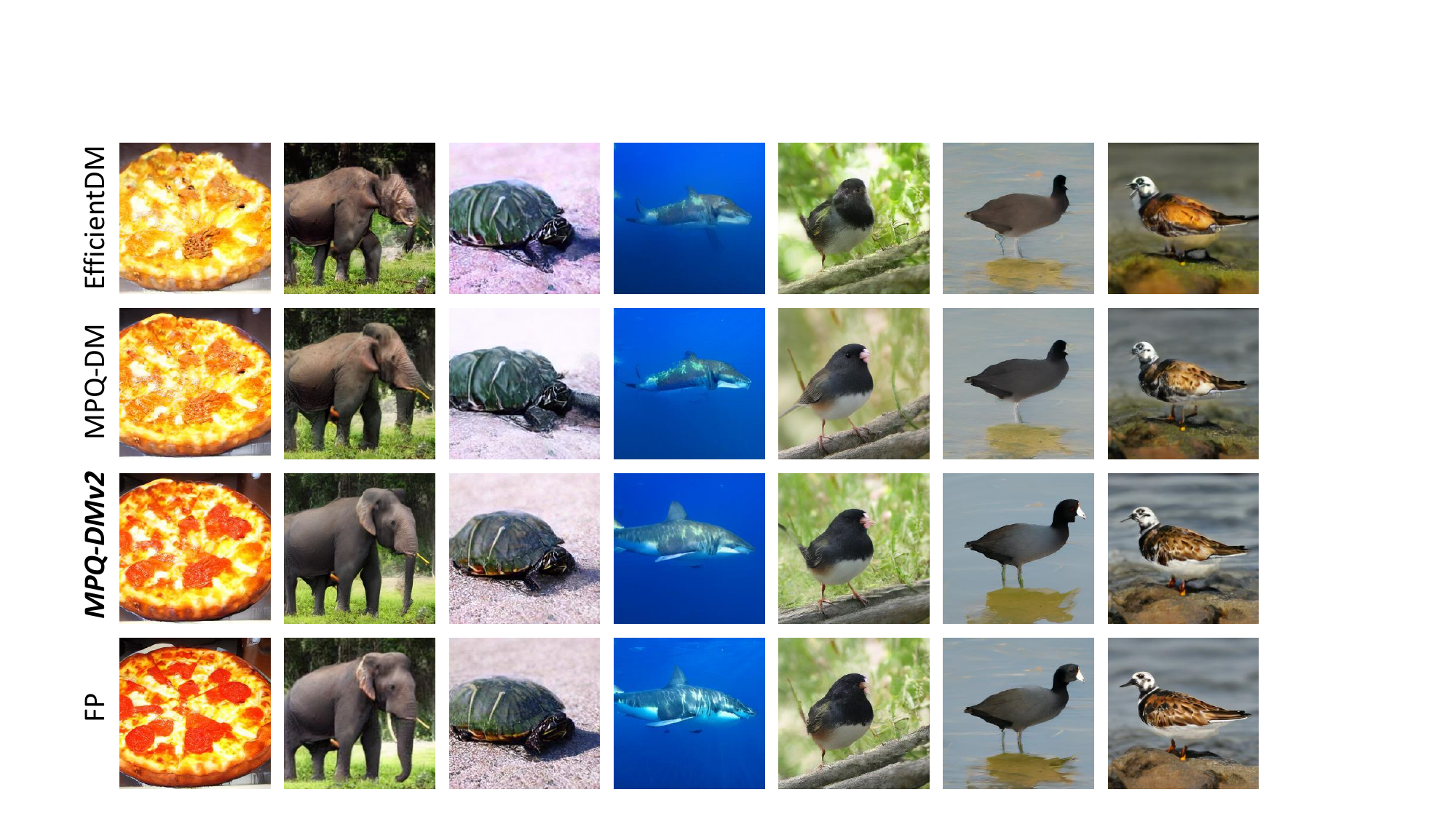}
    \caption{Visual comparison of different methods under W3A4 quantization setting on ImageNet 256$\times$256 LDM-4 model. Results of Full-Precision (FP) model are presented for better comparison.}
    \label{fig:visual_imagenet}
\end{figure*}

\begin{figure}
    \centering
    \includegraphics[width=1.0\linewidth]{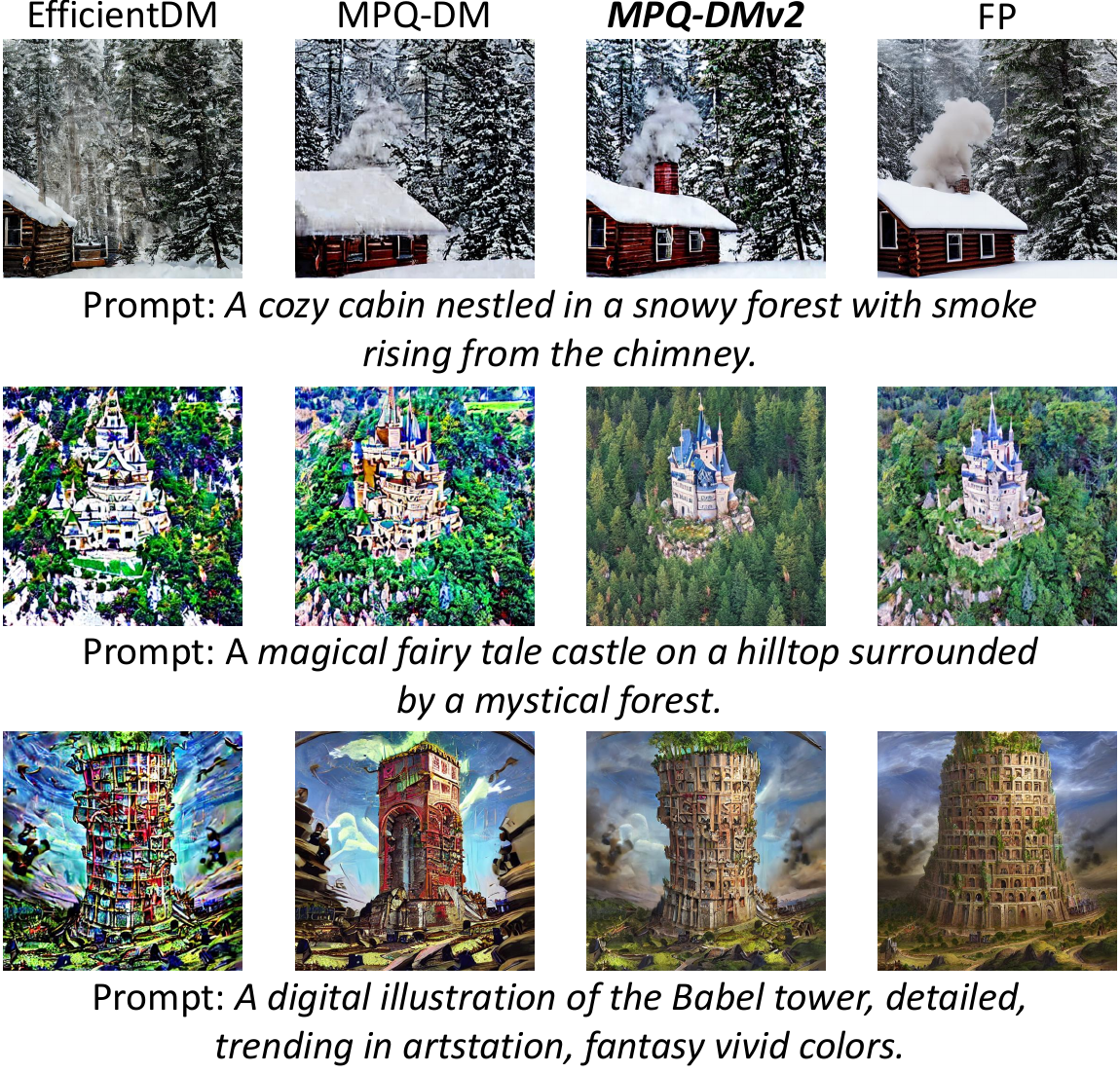}
    \caption{Visual comparison of different methods under W4A6 quantization setting on Stable Diffusion model.}
    \label{fig:visual_sd}
\end{figure}

\subsection{Ablation Study}
\textbf{Component Study.} In Table~\ref{tab:ablation_method}, we perform comprehensive ablation studies on LDM-4 ImageNet 256$\times$256 model to evaluate the effectiveness of each proposed component. Our proposed FZRMQ provides a flexible residual quantizer that is more suitable for low-bit quantization, making the quantization operation more suitable for outlier distributions in diffusion models, and optimizing the FID from 36.59 to 33.58. Furthermore, our optimization method MTRD preserves temporal information during the distillation process, ensuring the consistency of the overall denoising trajectory and further reducing FID to 33.02. Finally, we propose OOLRI to efficiently initialize the LORA module in the optimization framework using prior quantization error information. By combining the above three orthogonal techniques, our MPQ-DMv2 reduced the FID to 32.55 and achieved optimal performance.

\begin{table}[h]
    \centering
    \caption{Ablation study on proposed methods. The experiments are conducted on LDM-4 ImageNet 256$\times$256 under W2A4 quantization setting.}
    \setlength{\tabcolsep}{1.4mm}
    \begin{tabular}{lccccc}
    \hline
         Method & \makecell{Bit \\ (W/A)} & IS $\uparrow$ & FID $\downarrow$ & sFID $\downarrow$ & Precision $\uparrow$ \\ \hline
         EfficientDM~\cite{he2023efficientdm} & 2/4 & 25.20 & 64.45 & 14.99 & 36.63 \\ \hline
         MPQ-DM~\cite{feng2025mpqdm} (Baseline) & 2/4 & 43.95 & 36.59 & 12.20 & 52.14 \\ \hline 
         +FZRMQ (w/o FOS) & 2/4 & 45.28 & 34.95 & 12.70 & 52.83 \\
         +FZRMQ & 2/4 & 48.98 & 33.58 & 12.82 & 53.51 \\
         +MTRD (MSE) & 2/4 & 49.15 & 33.26 & 12.29 & 53.25 \\
         +MTRD (KL) & 2/4 & 49.33 & 33.02 & 12.14 & 53.74 \\
         \rowcolor[gray]{0.9}
         +OOLRI (\textbf{MPQ-DMv2}) & 2/4 & \textbf{49.79} & \textbf{32.55} & \textbf{12.12} & \textbf{54.40}\\ \hline
    \end{tabular}
    \label{tab:ablation_method}
\end{table}

\textbf{Ablation Study on Optimization Strategies.}
We further conduct experiments on the proposed \textit{Flexible Optimization Strategies} (FOS) mentioned in Sec.~\ref{subsec:quant_optimization} on Tab.~\ref{tab:ablation_method}. FOS aims to dynamically preserve the powerful expressive ability of high-bit uniform quantizers, aiming to achieve a balance between flexible handling of outliers and preserving uniform expressive power. It is worth mentioning that FOS only requires one additional search, without any additional training or inference overhead. Compared to not using FOS, the full version of FZRMQ reduced FID from 34.95 to 33.58, demonstrating the effectiveness of FOS.

\textbf{Ablation on Loss Metrics.}
In Table~\ref{tab:ablation_method}, we further present different distillation metrics used in Eq.~\eqref{eq:KL_loss} on Tab.~\ref{tab:ablation_method}. We mainly compare the used Kullback-Leibler (KL) divergence with Mean Squared Error (MSE) metrics. Compared to directly calculating the difference between each element using MSE, KL divergence can comprehensively consider the similarity between distributions~\cite{feng2024rdd, yang2022cirkd}. And our MTRD aims to distill the temporal information from the entire denoising trajectory, which is naturally a type of time distribution information. Therefore, we choose to use KL divergence in our implementation. The experiment results also proved that KL divergence can reduce FID from 33.26 to 33.02 compared to MSE, and also demonstrated its adaptability to our distillation target.

\begin{table}[h]
    \centering
    \caption{Ablation Study on Distillation Weight Factor. Experiment conducted on LDM-4 ImageNet 256$\times$256 model under W2A4 quantization settings. \textbf{``-'' denotes without distillation term.}}
    \setlength{\tabcolsep}{2.2mm}
    \begin{tabular}{c|cccccc}
    \hline
         $\alpha$ & - & 0.1 & 0.2 & 0.5 & \underline{1.0} & 2.0 \\ \hline
         FID $\downarrow$ & 33.58 & 33.24 & 33.28 & 33.16 & \textbf{33.02} & 33.20 \\ \hline
    \end{tabular}
    \label{tab:ablation_weight_factor}
\end{table}

\textbf{Ablation Study on Distillation Weight Factor.} In Tab.~\ref{tab:ablation_weight_factor}, we study different distillation weight factor $\alpha$ used in Eq.~\eqref{eq:total_loss}. To verify the weight robustness in MTRD, we set different $\alpha$ value to compare the FID score. It can be seen that all different $\alpha$ values have performance improvements compared to not distilling, which proves the robustness of our distillation method to the weight factor. Among them, $\alpha=1.0$ achieved the best FID value. Therefore, in our experiment, we chose $\alpha=1.0$ as our basic setting.

\begin{figure}[h]
    \centering
    \subfloat[][pixel queue size $L$]{
        \includegraphics[width=0.47\linewidth]{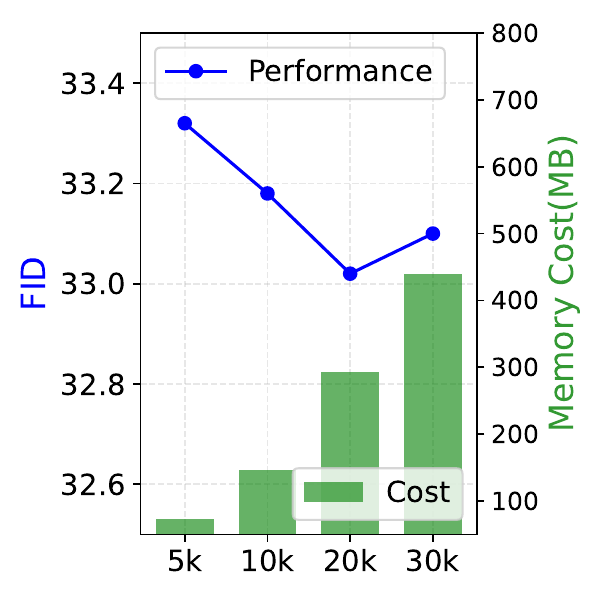}
        \label{fig:queue_size}
    }
    \subfloat[][queue sample size $k$]{
        \includegraphics[width=0.47\linewidth]{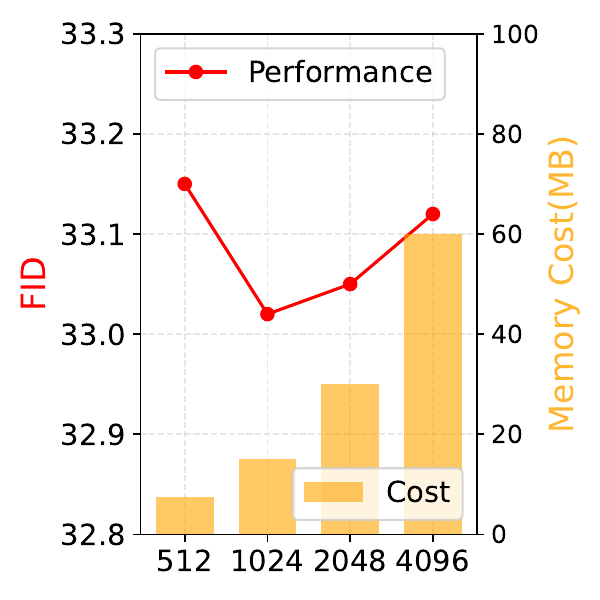}
        \label{fig:queue_sample}
    }
    \caption{Ablation study on memory-based online queue hyper-parameters. We conduct experiments on LDM-4 ImageNet 256$\times$256 model under W2A4 quantization setting.}
    \label{fig:ablation_queue}
\end{figure}

\textbf{Ablation Study on Memory-based Online Queue.} 
We present ablation study on the memory-based queue used in MTRD. We first investigate the queue size $L$ in Eq.~\eqref{eq:queue_size} and present the results in Fig.~\ref{fig:queue_size}. It can be seen that larger queue size increases the richness of features and improves performance, but reaches saturation after $L=20k$. Therefore, we chose $L=20k$ as our experimental setup. This online queue will only bring an additional 293MB of memory, which is memory-friendly. We then investigate the pixel sampling size $k$ in Eq.~\eqref{eq:queue_sample_size} and present the results in Fig.~\ref{fig:queue_sample}. It can be seen that the performance of the model improves with increasing sampling number when the sampling number is small, but gradually saturates after the optimal sampling number $k=1024$. Therefore, we use the optimal sampling number $k=1024$ in the experiment. And it is worth mentioning that this sampling strategy will only bring an additional 15MB of training memory, which is minor compared to the model size.

\begin{table}[h]
    \centering
    \caption{Inference efficiency comparison of calibration process on LDM-4 ImageNet 256$\times$256 model.}
    \setlength{\tabcolsep}{1.8mm}
    \begin{tabular}{l|cccc}
    \hline
         Method & \makecell{Time Cost \\ (hours)} & \makecell{GPU Memory \\(GB)} & FID $\downarrow$ & \makecell{Precision $\uparrow$ \\ (\%)} \\ \hline
         PTQD~\cite{he2024ptqd} & 2.98 & 15.6 & 336.57 & 0.01 \\
         QuEST~\cite{wang2024quest} & 3.05 & 20.3 & 285.42 & 0.03 \\
         EfficientDM~\cite{he2023efficientdm} & 3.02 & 19.4 & 64.45 & 36.63 \\
         MPQ-DM~\cite{feng2025mpqdm} & 3.12 & 19.7 & 36.59 & 52.14 \\
         \rowcolor[gray]{0.9}
         \textbf{MPQ-DMv2} & 3.13 & 19.8 & \textbf{32.55} & \textbf{54.40} \\ \hline
    \end{tabular}
    \label{tab:train_cost}
\end{table}

\begin{table}[h]
    \centering
    \caption{Efficiency evaluation on LDM-4 ImageNet 256$\times$256.}
    \begin{tabular}{l|cccc}
    \hline
         Method & \makecell{Bit \\ (W/A)} & Bops (T) & \makecell{Size (MB)} & Runtime (ms) \\ \hline
         FP~\cite{rombach2022ldm} & 32/32 & 102.3 & 1529.7 & 436.8 \\
         \rowcolor[gray]{0.9}
         \textbf{MPQ-DMv2} & 3/6 & \textbf{1.8} & \textbf{144.6 (10.58$\times$)} & \textbf{214.4 (2.04$\times$)}  \\
         \rowcolor[gray]{0.9}
         \textbf{MPQ-DMv2} & 2/4 & \textbf{0.8} & \textbf{96.8 (15.80$\times$)} & \textbf{130.4 (3.35$\times$)} \\ \hline
    \end{tabular}
    \label{tab:inference_cost}
\end{table}

\textbf{Ablation Study on Efficiency.}
We reported the efficiency of our MPQ-DMv2 in the calibration and inference process. We first present the calibration efficiency evaluation results in Tab.~\ref{tab:train_cost}. Compared with existing methods, MPQ-DMv2 significantly improves the quantization performance with minimal calibration burden. Especially compared with the original method MPQ-DM, MPQ-DMv2 further improves the model performance with almost no additional cost. We then present the inference efficiency in Tab.~\ref{tab:inference_cost}. We refer to the implementation of previous works~\cite{he2023efficientdm, he2024ptqd} to report the evaluation results. Compared with the FP model, the low-bit quantized MPQ-DMv2 significantly reduces computational complexity and model storage, proving the enormous potential of low-bit quantization for diffusion model inference efficiency.

\section{Conclusion}
In this paper, we proposed MPQ-DMv2, an improved mixed-precision quantization method for extremely low-bit diffusion models. To address 
the unfriendly quantization characteristics of the uniform quantizer for imbalanced outlier distribution. We propose Flexible Z-Order Residual Mixed Quantization to adjust quant steps for better representation ability. For the temporal inconsistency caused by separate quant parameters, we propose Memory-based Temporal Relation Distillation to use online time-aware pixel queue to construct temporal distribution for consistency distillation. To tackle the cold-start low-rank adaptation, we propose Object-Oriented Low-Rank Initialization to use prior quantization error for informative init. Our extensive experiments  demonstrated the superiority of MPQ-DMv2 over current SOTA methods on various model architectures and generation tasks.

\bibliography{aaai25}
\bibliographystyle{IEEEtran}




\end{document}